\documentclass{article}

\usepackage{times}
\usepackage{graphicx} 
\usepackage{subfigure}

\usepackage{natbib}

\usepackage{algorithm}
\usepackage{algorithmic}

\usepackage{hyperref}

\usepackage[accepted]{icml2017}

\usepackage{color}

\usepackage[utf8]{inputenc} 
\usepackage[T1]{fontenc}    
\usepackage{url}            
\usepackage{booktabs}       
\usepackage{amsfonts}       
\usepackage{nicefrac}       
\usepackage{dsfont}
\usepackage{thmtools,thm-restate}
\usepackage{amsmath,amssymb,amsthm}
\usepackage{here}
\usepackage{enumitem}
\graphicspath{{./media/}}

\newcommand{\Exp}{\mathds{E}}

\newcommand{\Prob}{\mathds{P}}
\newcommand{\Real}{\mathds{R}}

\newcommand{\Ind}{\mathds{1}}

\newcommand{\so}{\succcurlyeq_{\rm so}}

\newcommand{\ssd}{\succcurlyeq_{\rm ssd}}

\DeclareMathOperator*{\argmax}{arg\,max}

\newcommand{\Sc}{\mathcal{S}}
\newcommand{\Ac}{\mathcal{A}}
\newcommand{\Mc}{\mathcal{M}}

\newcommand{\Hc}{\mathcal{H}}

\newtheorem{theorem}{Theorem}

\newtheorem{lemma}{Lemma}
\newtheorem{definition}{Definition}

\newtheorem{conjecture}{Conjecture}

\newenvironment{hproof}{%
  \proof}{\endproof}

\icmltitlerunning{Why is Posterior Sampling Better than Optimism for Reinforcement Learning?}

\begin{document}

\twocolumn[
\icmltitle{Why is Posterior Sampling Better than Optimism for Reinforcement Learning?}



\icmlsetsymbol{equal}{*}

\begin{icmlauthorlist}
\icmlauthor{Ian Osband}{stan,dm}
\icmlauthor{Benjamin Van Roy}{stan}
\end{icmlauthorlist}

\icmlaffiliation{stan}{Stanford University, California, USA}
\icmlaffiliation{dm}{Deepmind, London, UK}

\icmlcorrespondingauthor{Ian Osband}{ian.osband@gmail.com}
\icmlkeywords{boring formatting information, machine learning, ICML}

\vskip 0.3in
]



\printAffiliationsAndNotice{} 

\begin{abstract}
{\hyphenpenalty=1000
\exhyphenpenalty=1000
Computational results demonstrate that \mbox{posterior} sampling for reinforcement learning (PSRL) \mbox{dramatically} outperforms existing algorithms driven by optimism, such as UCRL2.
We provide insight into the extent of this performance boost and the phenomenon that drives it.
We leverage this insight to establish an $\tilde{O}(H\sqrt{SAT})$ Bayesian regret bound for PSRL in finite-horizon episodic Markov decision processes.
This improves upon the best previous Bayesian regret bound of $\tilde{O}(H S \sqrt{AT})$ for any reinforcement learning algorithm.
Our theoretical results are supported by extensive empirical evaluation.
}
\end{abstract}

\section{Introduction}
\label{sec: intro}
We consider the reinforcement learning problem in which an agent interacts with a Markov decision process with the aim of maximizing expected cumulative reward \citep{burnetas1997optimal,Sutton1998}.
Key to performance is how the agent balances between exploration to acquire information of long-term benefit and exploitation to maximize expected near-term rewards.
In principle, dynamic programming can be applied to compute the Bayes-optimal solution to this problem \citep{bellman1959adaptive}.
However, this is computationally intractable for anything beyond the simplest of toy problems and direct approximations can fail spectacularly poorly \citep{munos2014bandits}.
As such, researchers have proposed and analyzed a number of heuristic reinforcement learning algorithms.

The literature on efficient reinforcement learning offers statistical efficiency guarantees for computationally tractable algorithms.
These provably efficient algorithms \citep{Kearns2002,Brafman2002}
predominantly address the exploration-exploitation trade-off via \textit{optimism in the face of uncertainty} (OFU): at any state, the agent assigns to each action an optimistically biased estimate of future value and selects the action with the greatest estimate.
If a selected action is not near-optimal, the estimate must be overly optimistic, in which case the agent learns from the experience.
Efficiency relative to less sophisticated exploration arises as the agent avoids actions that can neither yield high value nor informative data.

An alternative approach, based on Thompson sampling \citep{Thompson1933}, involves sampling a statistically plausibly set of action values and selecting the maximizing action.
These values can be generated, for example, by sampling from the posterior distribution over MDPs and computing the state-action value
function of the sampled MDP.  This approach, originally proposed in \citet{Strens00}, is called posterior sampling for reinforcement learning (PSRL).
Computational results from \citet{Osband2013} demonstrate that PSRL dramatically outperforms existing algorithms based on OFU.
The primary aim of this paper is to provide insight into the extent of this performance boost and the phenomenon that drives it.

We show that, in Bayesian expectation and up to constant factors, PSRL matches the statistical efficiency of \textit{any} standard algorithm for OFU-RL.
We highlight two key shortcomings of existing state of the art algorithms for OFU \cite{Jaksch2010} and demonstrate that PSRL does not suffer from these inefficiencies.
We leverage this insight to produce an $\tilde{O}(H\sqrt{SAT})$ bound for the Bayesian regret of PSRL in finite-horizon episodic Markov decision processes where $H$ is the horizon, $S$ is the number of states, $A$ is the number of actions and $T$ is the time elapsed.
This improves upon the best previous bound of $\tilde{O}(H S \sqrt{AT})$ for any RL algorithm.
We discuss why we believe PSRL satisfies a tighter $\tilde{O}(\sqrt{HSAT})$, though we have not proved that.
We complement our theory with computational experiments that highlight the issues we raise; empirical results match our theoretical predictions.

More importantly, we highlights a tension in OFU RL between statistical efficiency and computational tractability.
We argue that any OFU algorithm that matches PSRL in statistical efficiency would likely be computationally intractable.
We provide proof of this claim in a restricted setting.
Our key insight, and the potential benefits of exploration guided by posterior sampling, are not restricted to the simple tabular MDPs we analyze.


\section{Problem formulation}
\label{sec: problem formulation}

We consider the problem of learning to optimize a random finite-horizon MDP
{\medmuskip=0mu
\thinmuskip=0mu
\thickmuskip=0mu
$M^* = (\Sc, \Ac, R^* \hspace{-1mm}, P^*\hspace{-1mm}, H, \rho)$}
over repeated episodes of interaction, where
$\Sc =\{1,..,S\}$ is the state space, $\Ac=\{1,..,A\}$ is the action space, $H$ is the horizon, and $\rho$ is the initial state distribution.
In each time period $h=1,..,H$ within an episode, the agent observes state $s_h \in \Sc$, selects action $a_h \in \Ac$,
receives a reward $r_h \sim R^*(s_h,a_h)$, and transitions to a new state $s_{h+1} \sim P^*(s_h, a_h)$.
We note that this formulation, where the unknown MDP $M^*$ is treated as itself a random variable, is often called Bayesian reinforcement learning.

A policy $\mu$ is a mapping from state $s \in \Sc$ and period $h=1,..,H$ to action $a \in \Ac$.
For each MDP $M$ and policy $\mu$ we define the state-action value function for each period $h$:
{\medmuskip=2mu
\thinmuskip=1mu
\thickmuskip=2mu
\vspace{-4mm}
\begin{equation}
\label{eq: q value tabular}
  Q^{M}_{\mu, h}(s, a) := \Exp_{M,\mu}\left[ \sum_{j=h}^{H} \overline{r}^M(s_j,a_j) \Big| s_h = s, a_h=a \right],
\end{equation}
}
where
{\medmuskip=2mu
\thinmuskip=1mu
\thickmuskip=2mu
$\overline{r}^M(s,a) = \Exp[ r | r \sim R^M(s,a) ]$.
The subscript $\mu$ indicates that actions over periods $h+1,\ldots,H$ are selected according to the policy $\mu$.
Let $V^{M}_{\mu, h}(s) := Q^M_{\mu, h}(s, \mu(s,h))$.
We say a policy $\mu^M$ is optimal for the MDP $M$ if
\mbox{$\mu^M \in \argmax_{\mu} V^{M}_{\mu, h}(s)$ for all $s \in \Sc$ and $h=1,\ldots,H$}.

Let $\Hc_t$ denote the history of observations made \emph{prior} to time $t$.
To highlight this time evolution within episodes, with some abuse of notation, we let $s_{kh} = s_t$ for \mbox{$t=(k-1)H+h$}, so that
$s_{kh}$ is the state in period $h$ of episode $k$.
We define $\Hc_{kh}$ analogously.
An RL algorithm is a deterministic sequence $\{\pi_k | k = 1, 2, \ldots\}$ of functions, each mapping $\Hc_{k1}$ to a probability distribution $\pi_{k}(\Hc_{k1})$ over policies,
from which the agent samples a policy $\mu_k$ for the $k$th episode.
We define the regret incurred by an RL algorithm $\pi$ up to time $T$ to be
\vspace{-3mm}
\begin{equation}
\label{eq: regret}
    {\rm Regret}(T, \pi, M^*) := \sum_{k=1}^{\lceil T/H \rceil} \Delta_k,
\end{equation}
where $\Delta_k$ denotes regret over the $k$th episode, defined with respect to true MDP $M^*$ by
\begin{equation}
  \Delta_k := \sum_{\Sc} \rho(s) (V^{M^*}_{\mu^*, 1}(s) - V^{M^{*}}_{\mu_k, 1}(s))
\end{equation}
with $\mu^* = \mu^{M^*}$.
We note that the regret in \eqref{eq: regret} is random, since it depends on the unknown MDP $M^*$, the learning algorithm $\pi$ and through the history $\Hc_t$ on the sampled transitions and rewards.
We define
{\small
\begin{equation}
\label{eq: bayes_regret}
    {\rm BayesRegret}(T, \pi, \phi) := \Exp \left[{\rm Regret}(T, \pi, M^*) \mid M^* \sim \phi \right],
\end{equation}}
\noindent \hspace{-1.5mm} as the Bayesian expected regret for $M^*$ distributed according to the prior $\phi$.
We will assess and compare algorithm performance in terms of the regret and BayesRegret.

\subsection{Relating performance guarantees}
\label{sec:bayes_freq}

For the most part, the literature on efficient RL is sharply divided between the frequentist and Bayesian perspective \cite{vlassis2012bayesian}.
By volume, most papers focus on minimax regret bounds that hold with high probability for any $M^* \in \Mc$ some class of MDPs \cite{Jaksch2010}.
Bounds on the BayesRegret are generally weaker analytical statements than minimax bounds on regret.
A regret bound for any $M^* \in \Mc$ implies an identical bound on the BayesReget for any $\phi$ with support on $\Mc$.
A partial converse is available for $M^*$ drawn with non-zero probability under $\phi$, but does not hold in general \cite{Osband2013}.

Another common notion of performance guarantee is given by so-called ``sample-complexity'' or PAC analyses that bound the number of $\epsilon$-sub-optimal decisions taken by an algorithm \cite{Kakade2003,dann2015sample}.
In general, optimal bounds on regret $\tilde{O}(\sqrt{T})$ imply optimal bounds on sample complexity $\tilde{O}(\epsilon^{-2})$, whereas optimal bounds on the sample complexity give only an $\tilde{O}(T^{2/3})$ bound on regret \cite{osband2016thesis}.

Our formulation focuses on the simple setting on finite horizon MDPs, but there are several other problems of interest in the literature.
Common formulations include the discounted setting\footnote{Discount $\gamma = 1 - 1/H$ gives an effective horizon $O(H)$.} and problems with infinite horizon under some connectedness assumption \cite{Bartlett2009}.
This paper may contain insights that carry over to these settings, but we leave that analysis to future work.

Our analysis focuses upon Bayesian expected regret in finite horizon MDPs.
We find this criterion amenable to (relatively) simple analysis and use it obtain actionable insight to the design of practical algorithms.
We absolutely do not ``close the book'' on the exploration/exploitation problem - there remain many important open questions.
Nonetheless, our work may help to develop understanding within some of the outstanding issues of statistical and computational efficiency in RL.
In particular, we shed some light on how and why posterior sampling performs so much better than existing algorithms for OFU-RL.
Crucially, we believe that many of these insights extend beyond the stylized problem of finite tabular MDPs and can help to guide the design of practical algorithms for generalization and exploration via randomized value functions \cite{osband2016thesis}.

\section{Posterior sampling as stochastic optimism}
\label{sec: psrl_stoch_opt}

There is a well-known connection between posterior sampling and optimistic algorithms \cite{Russo2013}.
In this section we highlight the similarity of these approaches.
We argue that posterior sampling can be thought of as a \textit{stochastically} optimistic algorithm.

Before each episode, a typical OFU algorithm constructs a confidence set to represent the range of MDPs that are statistically plausible given prior knowledge and observations.
Then, a policy is selected by maximizing value simultaneously over policies and MDPs in this set.
The agent then follows this policy over the episode.
It is interesting to contrast this approach against PSRL where instead of maximizing over a confidence set, PSRL samples a single statistically plausible MDP and selects a policy to maximize value for that MDP.

{
\medmuskip=2mu
\thinmuskip=1mu
\thickmuskip=2mu
\begin{algorithm}[!h]
\caption{OFU RL}
\textbf{Input:} confidence set constructor $\Phi$
\begin{algorithmic}[1]
\label{alg:ofu_rl}
\FOR{episode $k=1, 2, ..$}
\STATE Construct confidence set $\Mc_k = \Phi(\Hc_{k1})$
\STATE Compute $\mu_k \in \argmax_{\mu, M \in \Mc_k} V^M_{\mu,1}$
\FOR{timestep $h=1,..,H$}
\STATE take action $a_{kh} = \mu_k(s_{kh}, h)$
\STATE update $H_{kh+1} = \Hc_{kh} \cup (s_{kh}, a_{kh}, r_{kh}, s_{kh+1})$
\ENDFOR
\ENDFOR
\end{algorithmic}
\end{algorithm}
\normalsize
}


{
\medmuskip=2mu
\thinmuskip=1mu
\thickmuskip=2mu
\begin{algorithm}[!h]
\caption{PSRL}
\textbf{Input:} prior distribution $\phi$
\begin{algorithmic}[1]
\label{alg:psrl}
\FOR{episode $k=1, 2, ..$}
\STATE Sample MDP $M_k \sim \phi(\cdot \mid \Hc_{k1})$
\STATE Compute $\mu_k \in \argmax_{\mu} V^{M_k}_{\mu,1}$
\FOR{timestep $h=1,..,H$}
\STATE take action $a_{kh} = \mu_k(s_{kh}, h)$
\STATE update $H_{kh+1} = \Hc_{kh} \cup (s_{kh}, a_{kh}, r_{kh}, s_{kh+1})$
\ENDFOR
\ENDFOR
\end{algorithmic}
\end{algorithm}
\normalsize
}

\vspace{-3mm}
\subsection{The blueprint for OFU regret bounds}
\label{sec: blueprint_ofu}

The general strategy for the analysis of optimistic algorithms follows a simple recipe \cite{strehl2005theoretical,szita2010model,munos2014bandits}:
\begin{enumerate}[leftmargin=*]
  \item Design confidence sets (via concentration inequality) such that $M^* \in \Mc_k $ for all $k$ with probability $\ge 1 - \delta$.

  \item Decompose the regret in each episode
  \begin{eqnarray*}
  \Delta_k = V^{M^*}_{\mu^*,1} - V^{M^*}_{\mu_k,1} = \underbrace{V^{M^*}_{\mu^*,1} - V^{M_k}_{\mu_k,1}}_{\Delta^{\rm opt}_k} + \underbrace{V^{M_k}_{\mu_k,1} - V^{M^*}_{\mu_k,1}}_{\Delta^{\rm conc}_k}
  \end{eqnarray*}
  where $M_k$ is the imagined optimistic MDP.

  \item By step (1.) $\Delta^{\rm opt}_k \le 0$ for all $k$ with probability $\ge 1 - \delta$.

  \item Use concentration results with a pigeonhole argument over all possible trajectories $\{\Hc_{11}, \Hc_{21}, ..\}$ to bound, with probability at least $1 - \delta$,
  {\small
  \medmuskip=1mu
\thinmuskip=0mu
\thickmuskip=1mu
  \begin{equation*}
    {\rm Regret}(T, \pi, M^*) \le \sum_{k=1}^{\lceil T / H \rceil} \Delta^{\rm conc}_k \mid M^* \in \Mc_k \le f(S,A,H,T,\delta).
  \end{equation*}}
\end{enumerate}

\subsection{Anything OFU can do, PSRL can expect to do too}
\label{sec: anything_psrl}

In this section, we highlight the connection between posterior sampling and any optimistic algorithm in the spirit of Section \ref{sec: blueprint_ofu}.
Central to our analysis will be the following notion of stochastic optimism \cite{osband2014generalization}.

\begin{definition}[Stochastic optimism]
\label{def: optimism}
\hspace{0.0001mm} \newline
Let $X$ and $Y$ be real-valued random variables with finite expectation.
We will say that $X$ is stochastically optimistic for $Y$ if for any convex and increasing $u:\Real \rightarrow \Real$:
\begin{equation}
\label{eq: optimism}
    \Exp\left[ u(X) \right] \ge \Exp\left[ u(Y) \right].
\end{equation}
We will write $X \so Y$ for this relation.
\end{definition}

This notion of optimism is dual to second order stochastic dominance \cite{hadar1969rules}, $X \so Y$ if and only if $-Y \ssd -X$.
We say that PSRL is a stochastically optimistic algorithm since the random \textit{imagined} value function $V^{M_k}_{\mu_k,1}$ is stochastically optimistic for the true optimal value function $V^{M^*}_{\mu^*, 1}$ conditioned upon \textit{any} possible history $\Hc_{k1}$ \cite{Russo2013}.
This observation leads us to a general relationship between PSRL and the BayesRegret of \textit{any} optimistic algorithm.

\vspace{2mm}
\begin{theorem}[PSRL matches OFU-RL in BayesRegret]
\label{thm: psrl best optimistic}
\hspace{0.000001mm} \newline
Let $\pi^{\rm opt}$ be any optimistic algorithm for reinforcement learning in the style of Algorithm \ref{alg:ofu_rl}.
If $\pi^{\rm opt}$ satisfies regret bounds such that, for any $M^*$ any $T > 0$ and any $\delta > 0$ the regret is bounded with probability at least $1 - \delta$
{\small
\begin{equation}
\label{eq: opt_regret_thm}
    {\rm Regret}(T, \pi^{\rm opt}, M^*) \le f(S,A,H,T, \delta).
\end{equation}
}
Then, if $\phi$ is the distribution of the true MDP $M^*$ and the proof of \eqref{eq: opt_regret_thm} follows Section \ref{sec: blueprint_ofu}, then for all $T > 0$
{\small
\begin{equation}
    {\rm BayesRegret}(T, \pi^{\rm PSRL}, \phi) \le
    2 f(S,A,H,T,{\delta \hspace{-1mm} = \hspace{-1mm} T^{-1}}) + 2.
\end{equation}
}
\end{theorem}

\begin{hproof}
This result is established in \citet{Osband2013} for the special case of $\pi^{\rm opt} = \pi^{\rm UCRL2}$.
We include this small sketch as a refresher and a guide for high level intuition.
First, note that conditioned upon any data $\Hc_{k1}$, the true MDP $M^*$ and the sampled $M_k$ are identically distributed.
This means that $\Exp[\Delta^{\rm opt} | \Hc_{k1} ] \le 0$ for all $k$.
Therefore, to establish a bound upon the Bayesian regret of PSRL,  we just need to bound $\sum_{k=1}^{\lceil T / H \rceil} \Exp[ \Delta^{\rm conc}_k \mid \Hc_k ]$.

We can use that $M^* \mid \Hc_{k1} =^D M_k \mid \Hc_{k1}$ again in step (1.) from Section \ref{sec: blueprint_ofu} to say that \textit{both} $M^*, M_k$ lie within $\Mc_k$ for all $k$ with probability at least $1 - 2 \delta$ via a union bound.
This means we can bound the concentration error in PSRL,
{\small
\medmuskip=1mu
\thinmuskip=0mu
\thickmuskip=1mu
$$ {\rm BayesRegret}(T, \pi^{\rm PSRL}, \phi) \le \sum_{k=1}^{\lceil T/H\rceil} \Exp[\Delta^{\rm conc}_k \mid M^*, M_k \in \Mc_k ] + 2 \delta T$$
}
The final step follows from decomposing $\Delta^{\rm conc}_k$ by adding and subtracting the imagined optimistic value $\tilde{V}_k$ generated by $\pi^{\rm opt}$.
Through an application of the triangle inequality, $\Delta^{\rm conc}_k \le |V^{M_k}_{\mu_k,1} - \tilde{V}_k | + |\tilde{V}_k - V^{*}_{\mu_k,1}| $
we can mirror step (4.) to bound the regret from concentration,
\mbox{$\sum_{k=1}^{\lceil T/H\rceil} \Exp[\Delta^{\rm conc}_k \mid M^*, M_k \in \Mc_k ] \le 2f(S,A,H,T,\delta).$}
This result (and proof strategy) was established in multi-armed bandits by \citet{Russo2013}.
We complete the proof of Theorem \ref{thm: psrl best optimistic} with the choice $\delta = T^{-1}$ and that the regret is uniformly bounded by $T$.
\end{hproof}

\vspace{-4mm}

Theorem \ref{thm: psrl best optimistic} suggest that, according to Bayesian expected regret, PSRL performs within a factor of $2$ of any optimistic algorithm whose analysis follows Section \ref{sec: blueprint_ofu}.
This includes the algorithms UCRL2 \cite{Jaksch2010}, UCFH \cite{dann2015sample}, MORMAX \cite{szita2010model} and many more.

Importantly, and unlike existing OFU approaches, the algorithm performance is separated from the analysis of the confidence sets $\Mc_k$.
This means that PSRL even attains the big $O$ scaling of as-yet-undiscovered approaches to OFU, all at a computational cost no greater than solving a single known MDP - even if the matched OFU algorithm $\pi^{\rm opt}$ is computationally intractable.

\section{Some shortcomings of existing OFU-RL}
\label{sec: problem optimism}

In this section, we discuss how and why existing OFU algorithms forgo the level of statistical efficiency enjoyed by PSRL.
At a high level, this lack of statistical efficiency emerges from sub-optimal construction of the confidence sets $\Mc_k$.
We present several insights that may prove crucial to the design of improved algorithms for OFU.
More worryingly, we raise the question that perhaps the \textit{optimal} statistical confidence sets $\Mc_k$ would likely be computationally intractable.
We argue that PSRL offers a computationally tractable approximation to this unknown ``ideal'' optimistic algorithm.

Before we launch into a more mathematical argument it is useful to take intuition from a simple estimation problem, without any decision making.
Consider an MDP with $A=1, H=2, S=2N+1$ as described in Figure \ref{fig: bandit_s}.
Every episode the agent transitions from $s=0$ uniformly to $s \in \{1,..,2N\}$ and receives a deterministic reward from $\{0, 1\}$ depending upon this state.
The simplicity of these examples means even a naive monte-carlo estimate of the value should concentrate $1/2 \pm \tilde{O}(1 / \sqrt{n})$ after $n$ episodes of interaction.
Nonetheless, the confidence sets suggested by state of the art OFU-RL algorithm UCRL \citep{Jaksch2010} become incredibly mis-calibrated as $S$ grows.

\begin{figure}[h!]
\centering
  \includegraphics[width=.55\linewidth]{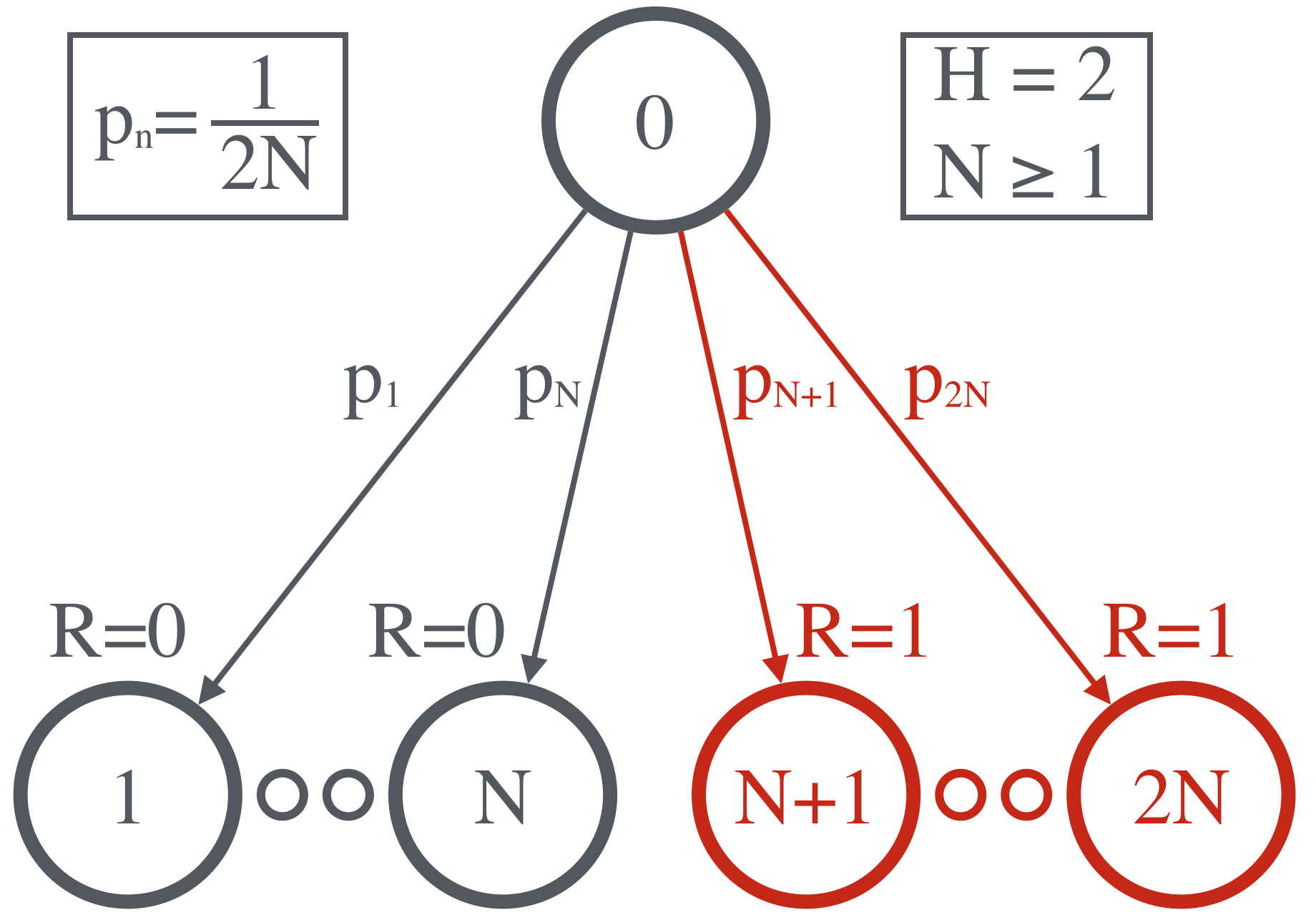}
\caption{MDPs to illustrate the scaling with $S$.}
\label{fig: bandit_s}
\end{figure}

\begin{figure}[h!]
\centering
  \includegraphics[width=.55\linewidth]{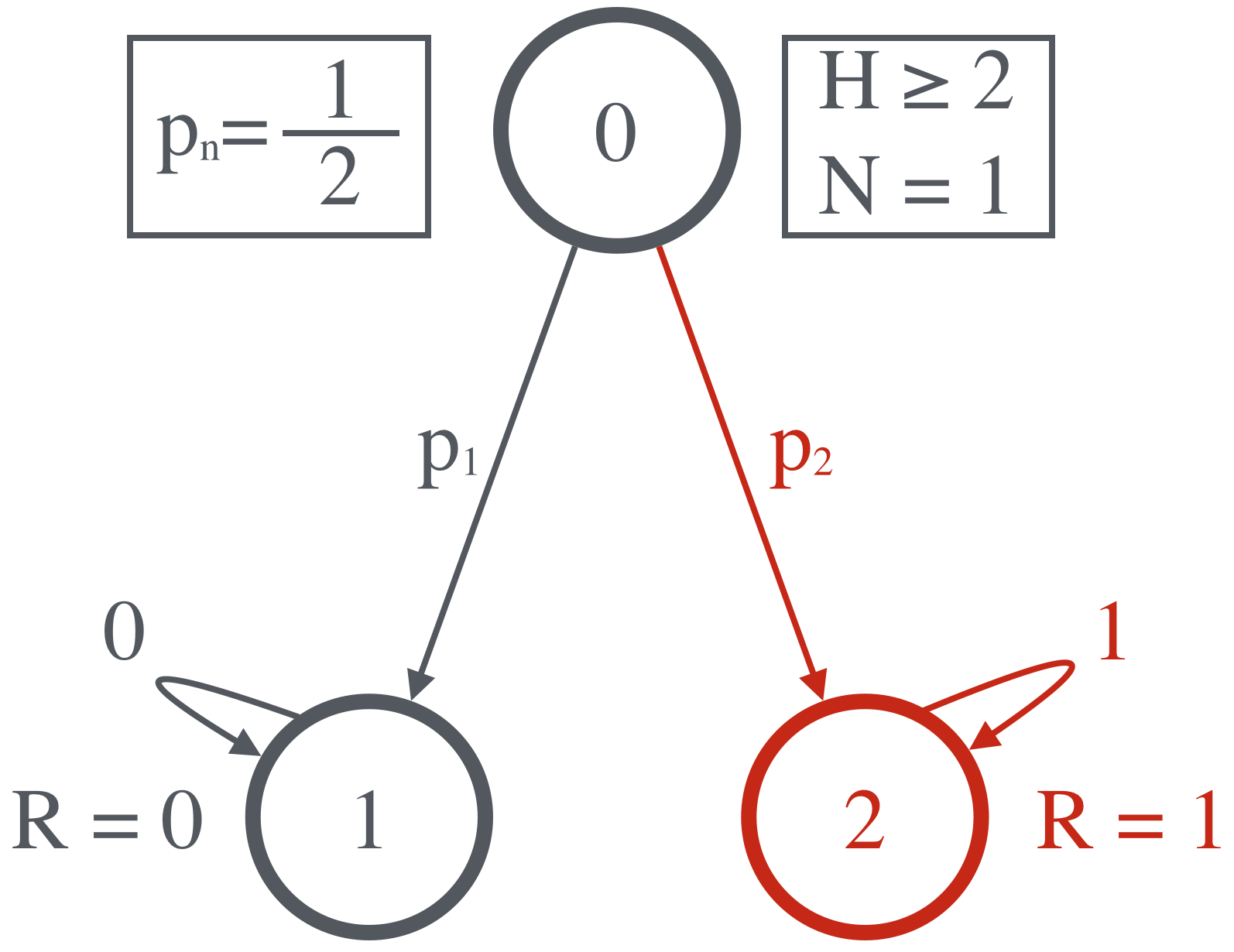}
\caption{MDPs to illustrate the scaling with $H$.}
\label{fig: bandit_h}
\end{figure}

\begin{figure}[h!]
\centering
  \includegraphics[width=.99\linewidth]{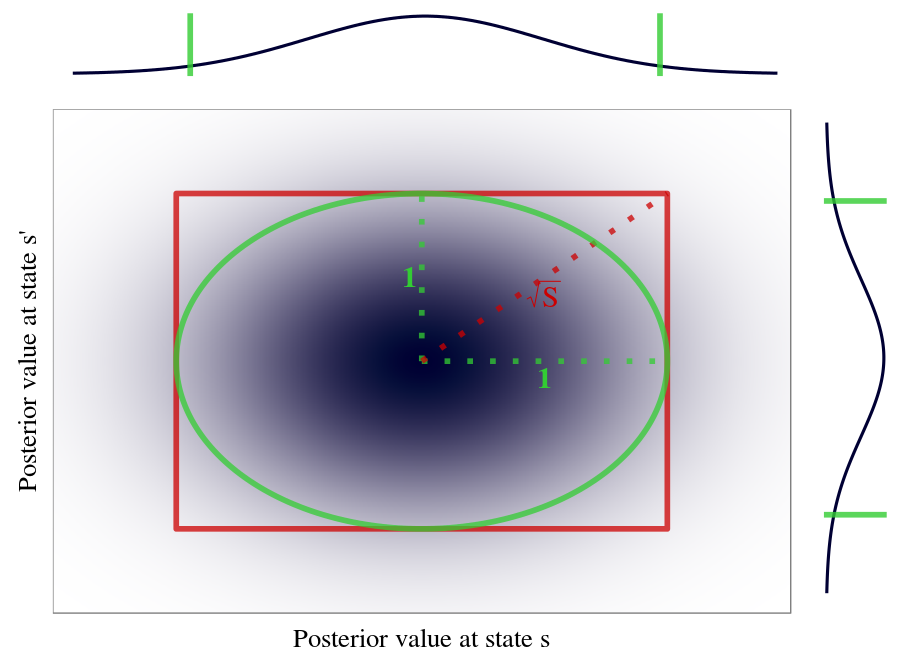}
\caption{\small Union bounds give loose rectangular confidence sets.}
\label{fig: rect conf}
\end{figure}

To see how this problem occurs, consider any algorithm for for model-based OFU-RL that builds up confidence sets for each state and action independently, such as UCRL.
Even if the estimates are tight in each state and action, the resulting optimistic MDP, simultaneously optimistic across each state and action, may be far too optimistic.
Geometrically these independent bounds form a rectangular confidence set.
The corners of this rectangle will be $\sqrt{S}$ misspecified to the underlying distribution, an ellipse, when combined across $S$ independent estimates (Figure \ref{fig: rect conf}).

Several algorithms for OFU-RL do exist which address this loose dependence upon $S$ \citep{Strehl2006,szita2010model}.
However, these algorithms depend upon a partitioning of data for future value, which leads to a poor dependence upon the horizon $H$ or equivalently the effective horizon $\frac{1}{1-\gamma}$ in discounted problems.
We can use a similar toy example from Figure \ref{fig: bandit_h} to understand why combining independently optimistic estimates through time will contribute to a loose bound in $H$.

The natural question to ask is, ``Why don't we simply apply these observations to design an optimistic algorithm which is simultaneously efficient in $S$ and $H$?''.
The first impediment is that designing such an algorithm requires some new intricate concentration inequalities and analysis.
Doing this rigorously may be challenging, but we believe it will be possible through a more careful application of existing tools to the insights we raise above.
The bigger challenge is that, even if one were able to formally specify such an algorithm, the resulting algorithm may in general not be computationally tractable.

A similar observation to this problem of optimistic optimization has been shown in the setting of linear bandits \cite{DaniHK2008,Russo2013}.
In these works they show that the problem of efficient optimization over ellipsoidal confidence sets can be NP-hard.
This means that computationally tractable implementations of OFU have to rely upon inefficient rectangular confidence sets that give up a factor of $\sqrt{D}$ where $D$ is the dimension of the underlying problem.
By contrast, Thompson sampling approaches remain computationally tractable (since they require solving only a single problem instance) and so do not suffer from the loose confidence set construction.
It remains an open question whether such an algorithm can be designed for finite MDPs.
However, these previous results in the simpler bandit setting $H=1$ show that these problems with OFU-RL cannot be overcome in general.

\vspace{-1mm}
\subsection{Computational illustration}
\label{sec: bandit_computation}
\vspace{-1mm}

In this section we present a simple series of computational results to demonstrate this looseness in both $S$ and $H$.
We sample $K=1000$ episodes of data from the MDP and then examine the optimistic/sampled Q-values for UCRL2 and PSRL.
We implement a version of UCRL2 optimized for finite horizon MDPs and implement PSRL with a uniform Dirichlet prior over the initial dynamics $P(0,1)=(p_1,..,p_{2N})$ and a $N(0,1)$ prior over rewards updating as if rewards had $N(0,1)$ noise.
For both algorithms, if we say that $R$ or $P$ are \textit{known} then we mean that we use the true $R$ or $P$ inside UCRL2 or PSRL.
In each experiment, the estimates guided by OFU become extremely mis-calibrated, while PSRL remains stable.

The results of Figure \ref{fig: bandit_knownP} are particularly revealing.
They demonstrates the potential pitfalls of OFU-RL even when the underlying transition dynamics \textit{entirely known}.
Several OFU algorithms have been proposed to remedy the loose UCRL-style L1 concentration from transitions \citep{filippi2010optimism,araya2012near,dann2015sample} but none of these address the inefficiency from hyper-rectangular confidence sets.
As expected, these loose confidence sets lead to extremely poor performance in terms of the regret.
We push full results to Appendix \ref{app: estimation_experiments} along with comparison to several other OFU approaches.

\begin{figure}[h!]
\centering
  \includegraphics[width=.75\linewidth, height=1.45in]{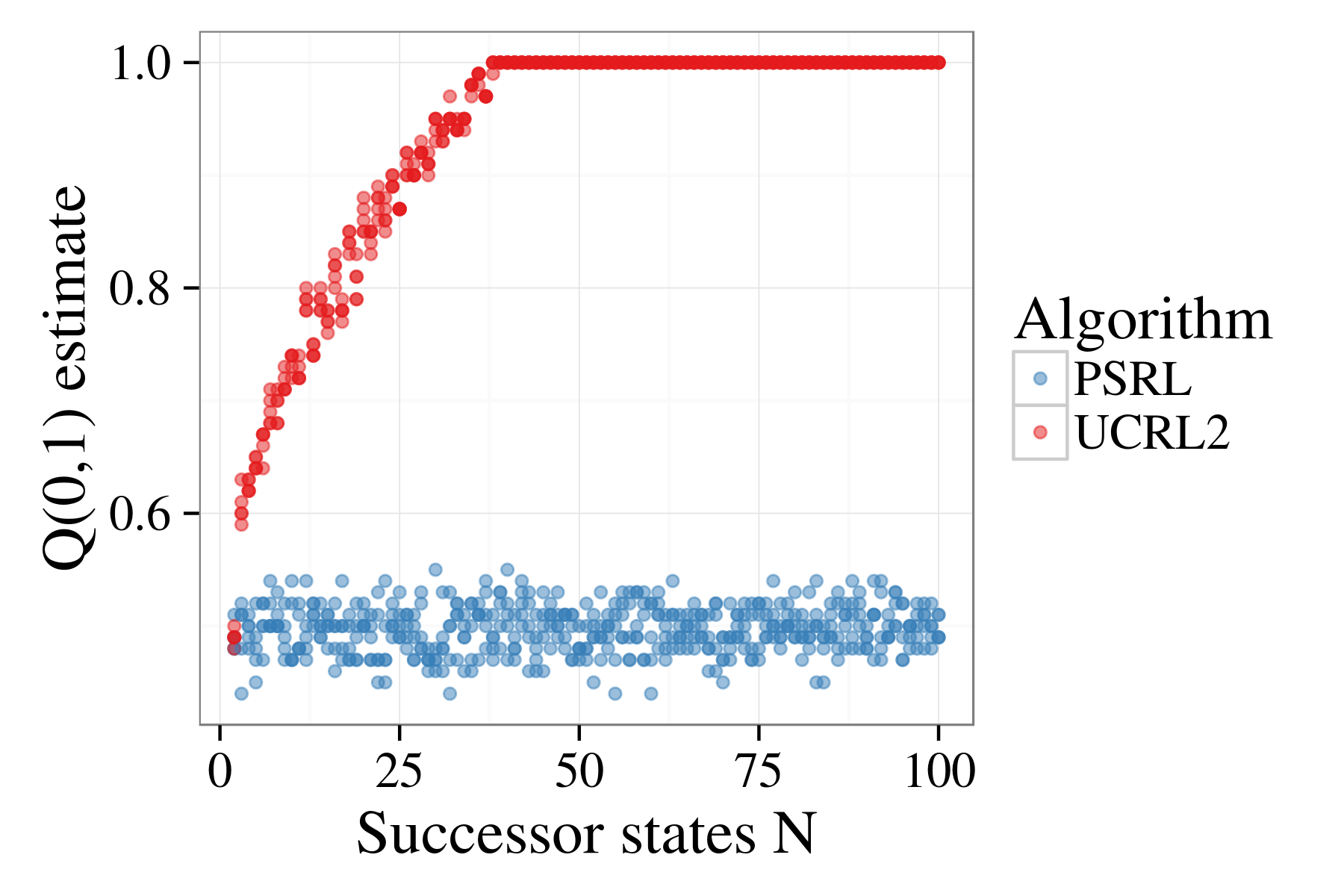}
  \vspace{-5mm}
\caption{$R$ known, $P$ unknown, vary $N$ in the MDP Figure \ref{fig: bandit_s}.}
\label{fig: bandit_knownR}
\end{figure}

\begin{figure}[h!]
\centering
  \includegraphics[width=.75\linewidth, height=1.45in]{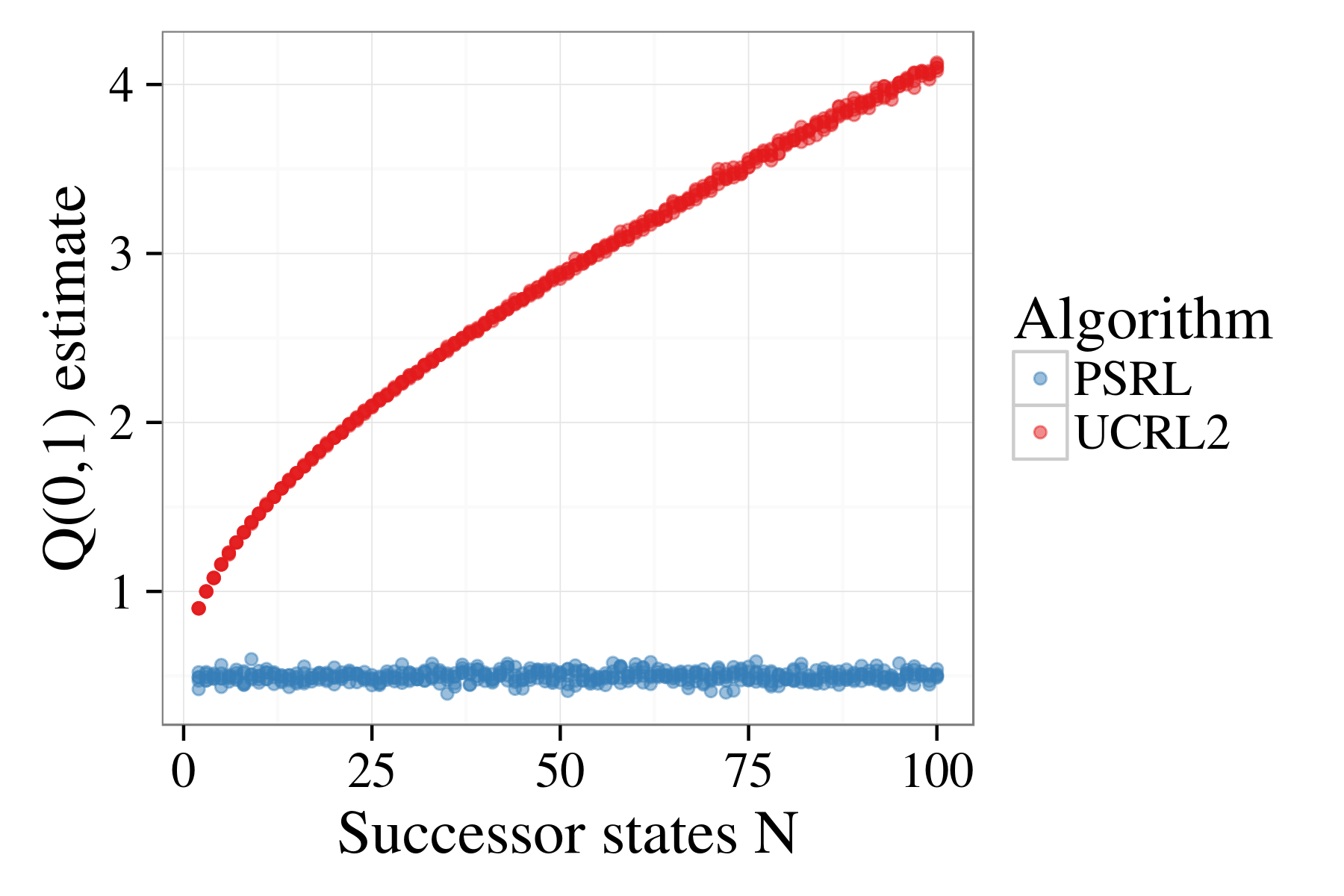}
  \vspace{-5mm}
\caption{$P$ known, $R$ unknown, vary $N$ in the MDP Figure \ref{fig: bandit_s}.}
\label{fig: bandit_knownP}
\end{figure}

\begin{figure}[h!]
\centering
  \includegraphics[width=.75\linewidth, height=1.45in]{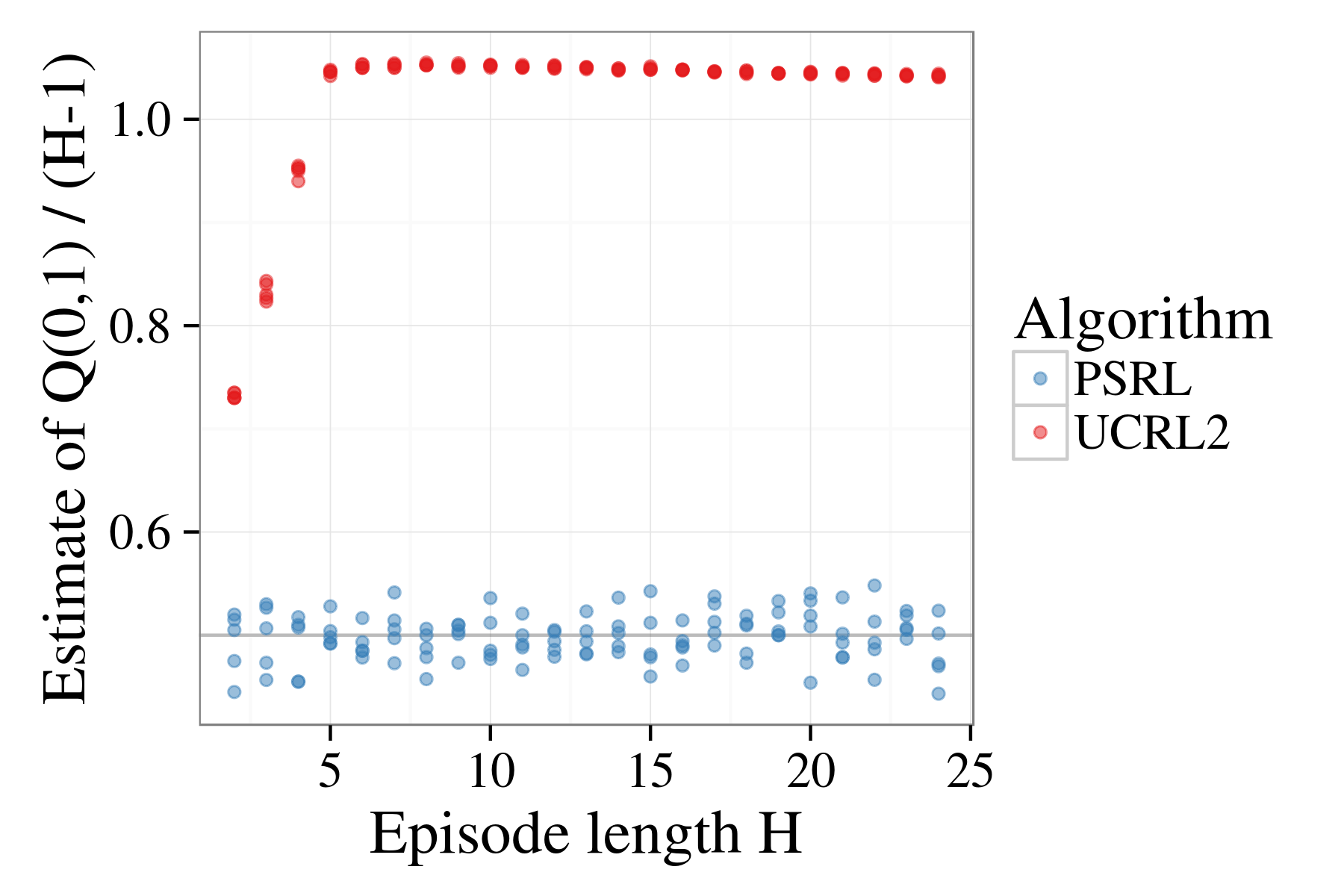}
  \vspace{-5mm}
\caption{ $R, P$ unknown, vary $H$ in the MDP Figure \ref{fig: bandit_h}}
\label{fig: bandit_epLen}
\end{figure}

\section{Better optimism by sampling}
\label{sec: stochastic optimism}

Until now, all analyses of PSRL have come via comparison to some existing algorithm for OFU-RL.
Previous work, in the spirit of Theorem \ref{thm: psrl best optimistic}, leveraged the existing analysis for UCRL2 to establish an $\tilde{O}(HS \sqrt{AT})$ bound upon the Bayesian regret \cite{Osband2013}.
In this section, we present a new result that bounds the expected regret of PSRL $\tilde{O}(H \sqrt{SAT})$.
We also include a conjecture that improved analysis could result in a Bayesian regret bound $\tilde{O}(\sqrt{HSAT})$ for PSRL, and that this result would be unimprovable \citep{osband2016lower}.

\subsection{From $S$ to $\sqrt{S}$}

In this section we present a new analysis that improves the bound on the Bayesian regret from $S$ to $\sqrt{S}$.
The proof of this result is somewhat technical, but the essential argument comes from the simple observation of the loose rectangular confidence sets from Section \ref{sec: problem optimism}.
The key to this analysis is a technical lemma on Gaussian-Dirichlet concentration \cite{osband2017gaussian}.
\begin{theorem}
\label{thm: psrl tight}
Let $M^*$ be the true MDP distributed according to prior $\phi$ with any independent Dirichlet prior over transitions.
Then the regret for PSRL is bounded
\begin{equation}
    {\rm BayesRegret}(T, \pi^{\rm PSRL}, \phi) = \tilde{O}\left(H \sqrt{SAT}\right).
\end{equation}
\end{theorem}

Our proof of Theorem \ref{thm: psrl tight} mirrors the standard OFU-RL analysis from Section \ref{sec: blueprint_ofu}.
To condense our notation we write
{\medmuskip=0mu
\thinmuskip=0mu
\thickmuskip=0mu
\mbox{$x_{kh} := (s_{kh}, a_{kh})$} and $V^k_{k,h} := V^{M_k}_{\mu_k,h}$.
Let the posterior mean of rewards $\hat{r}_k(x) := \Exp[ \overline{r}^*(x) | \Hc_{k1}]$, transitions $ \hat{P}_k(x) := \Exp[ P^*(x) | \Hc_{k1}]$ with respective deviations from sampling noise
\mbox{$w^R(x) := \overline{r}_k(x) - \hat{r}_k(x)$}
and \mbox{$w^P_h(x) := (P_k(x) - \hat{P}_k(x))^T V^k_{kh+1}$}.
}

We note that, conditional upon the data $\Hc_{k1}$ the true reward and transitions are independent of the rewards and transitions sampled by PSRL, so that $\Exp[\overline{r}^*(x) | \Hc_{k1}] = \hat{r}_k(x), \Exp[P^*(x) | \Hc_{k1}] = \hat{P}_k(x)$ for any $x$.
However, $\Exp[w^R(x) | \Hc_{k1}]$ and $\Exp[w^P_h(x) | \Hc_{k1}]$ are generally non-zero, since the agent chooses its policy to optimize its reward under $M_k$.
We can rewrite the regret from concentration via the Bellman operator (section 5.2 of \citet{Osband2013}),
{\small
\medmuskip=0mu
\thinmuskip=0mu
\thickmuskip=0mu
\begin{eqnarray}
\label{eq: regret_decomp}
    && \Exp\left[ V^k_{k1} - V^*_{k1} \mid \Hc_{k1}\right] \nonumber \\
    &=& \Exp\left[(\overline{r}_k - \overline{r}^*)(x_{k1}) + P_k(x_{k1})^T V^k_{k2} - P^*(x_{k1})^T V^*_{k2} \ \mid \ \Hc_{k1} \right] \nonumber \\
    &=& \Exp\bigg[(\overline{r}_k - \overline{r}^*)(x_{k1}) + \left(P_k(x_{k1}) - \hat{P}_k(x_{k1})\right)^T V^k_{k2} \nonumber \\
    &&+ \ \Exp \left[\left(V^k_{k2} - V^*_{k2}\right)(s') \mid s' \sim P^*(x_{k1})\right] \ \mid \Hc_{k1} \bigg] \nonumber \\
    &=& ... \nonumber \\
    &=& \Exp\bigg[ \textstyle \sum_{h=1}^H \left\{\overline{r}_k(x_{k1}) - \hat{r}^*(x_{k1}) \right\} \nonumber \\
    && + \textstyle \sum_{h=1}^H \left\{ \left(P_k(x_{kh}) - \hat{P}_k(x_{kh})\right)^T V^k_{kh} \right\} \ \mid \ \Hc_{k1} \bigg] \nonumber \\
    &\le& \Exp\left[ \textstyle \sum_{h=1}^H | w^R(x_{kh}) | + \textstyle \sum_{h=1}^H | w^P_h(x_{kh}) | \ \mid \ \Hc_{k1} \right].
\end{eqnarray}
}

\vspace{-4mm}

We can bound the contribution from unknown rewards $w^R_k(x_{kh})$ with a standard argument from earlier work \citep{buldygin1980sub,Jaksch2010}.

\begin{lemma}[Sub-Gaussian tail bounds]
\label{lem: tail}
\hspace{0.000000001mm} \newline
Let $x_1, .., x_n$ be independent samples from sub-Gaussian random variables.
Then, for any $\delta >0$
\begin{equation}
\label{eq: tail_2}
    \Prob \left( \frac{1}{n} \big| \sum_{i=1}^n x_i \big| \ge \sqrt{\frac{2 \log(2 / \delta)}{n}} \right)
        \le \delta.
\end{equation}
\end{lemma}

The key piece of our new analysis will be to show that the contribution from the transition estimate $\sum_{h=1}^H | w^P(x_{kh}) |$ concentrates at a rate independent of $S$.
At the root of our argument is the notion of stochastic optimism \citep{osband2016thesis}, which introduces a partial ordering over random variables.
We make particular use of Lemma \ref{lem:Gaussian-Dirichlet}, that relates the concentration of a Dirichlet posterior with that of a matched Gaussian distribution \cite{osband2017gaussian}.

\begin{lemma}[Gaussian-Dirichlet dominance]
\label{lem:Gaussian-Dirichlet}
\hspace{0.001mm} \newline
For all fixed $V \in [0,1]^N$, $\alpha \in [0,\infty)^N$ with \mbox{$\alpha^T \Ind \ge 2$},
if $X \sim N(\alpha^\top V / \alpha^\top {\bf 1}, 1/ \alpha^\top {\bf 1})$
and $Y = P^T V$ for \mbox{$P \sim {\rm Dirichlet}(\alpha)$} then $X \so Y$.
\end{lemma}

\vspace{-2mm}

We can use Lemma \ref{lem:Gaussian-Dirichlet} to establish a similar concentration bound on the error from sampling $w^P_h(x)$.

\begin{restatable}[Transition concentration]{lemma}{transconc}
\label{lem: transition conc psrl}
For any independent prior over rewards with $\overline{r} \in [0,1]$, additive sub-Gaussian noise and an independent Dirichlet prior over transitions at state-action pair $x_{kh}$, then
\begin{equation}
    w^P_h(x_{kh}) \le 2 H \sqrt{\frac{2 \log(2 / \delta)}{\max(n_k(x_{kh})-2, 1)}}
\end{equation}
with probability at least $1 - \delta$.
\end{restatable}

\vspace{-2mm}
\begin{hproof}
Our proof relies heavily upon some technical results from the note from \citet{osband2017gaussian}.
We cannot apply Lemma \ref{lem:Gaussian-Dirichlet} directly to $w^P$, since the future value $V^k_{kh+1}$ is itself be a random variable whose value depends on the sampled transition $P_k(x_{kh})$.
However, although $V^k_{kh+1}$ can vary with $P_k$, the structure of the MDP means that resultant $w^P(x_{kh})$ is still no more optimistic than the most optimistic possible \textit{fixed} $V \in [0,H]^S$.

We begin this proof only for the simply family of MDPs with $S=2$, which we call $\Mc_2$.
We write $p := P_k(x_{kh})(1)$ for the first component of the unknown transition at $x_{kh}$ and similarly $\hat{p} := \hat{P}_k(x_{kh})(1)$.
We can then bound the transition concentration,
{\small
\begin{eqnarray}
\label{eq: w_trans_beta}
| w^P_h(x_{kh}) | &=& | (P_k(x_{kh}) - \hat{P}_k(x_{kh}))^T V^k_{k h+1} | \nonumber \\
    &\le& | (p - \hat{p})| |(V^k_{k h+1}(1) - V^k_{k h+1}(2)) | \nonumber \\
    &\le& | p - \hat{p} | \sup_{R_k, P_k} | (V^k_{k h+1}(1) - V^k_{k h+1}(2)) | \nonumber \\
    &\le& | (p - \hat{p}) | H
\end{eqnarray}
}
Lemma \ref{lem:Gaussian-Dirichlet} now implies that for any $\alpha \in \Real_+$ with $\alpha^T \Ind \ge 2$, the random variables $p \sim {\rm Dirichlet}(\alpha)$  and $X \sim N(0, \sigma^2 = 1 / \alpha^T \Ind)$ are ordered,
\begin{equation}
\label{eq: X_trans_beta}
    X \so p - \hat{p} \ \implies | X | H \so | p - \hat{p} | H \so | w^P_h(x_{kh})|.
\end{equation}
We conclude the proof for $M \in \Mc_2$ through an application of Lemma \ref{lem: tail}.
To extend this argument to multiple states $S > 2$ we consider the marginal distribution of $P_k$ over any subset of states, which is Beta distributed similar to \eqref{eq: w_trans_beta}.
We push the details to Appendix \ref{app: proof_trans}.
\end{hproof}

To complete the proof of Theorem \ref{thm: psrl tight} we combine Lemma \ref{lem: tail} with Lemma \ref{lem: transition conc psrl}.
We rescale $\delta \leftarrow \delta / 2SAT$ so that these confidence sets hold at each $R(s,a), P(s,a)$ via union bound with probability at least $1 - \frac{1}{T}$,
\begin{eqnarray}
\label{eq: deviations}
   && \Exp\left[ \textstyle \sum_{h=1}^H \left\{ | w^R(x_{kh}) | + | w^P_h(x_{kh}) | \right\} \ \mid \Hc_{k1} \right] \nonumber \\
   &\le&
   \textstyle \sum_{h=1}^H 2 \left(H + 1\right) \sqrt{\frac{2 \log(4SAT)}{\max(n_k(x_{kh}) -2, 1)}}.
\end{eqnarray}

We can now use \eqref{eq: deviations} together with a pigeonhole principle over the number of visits to each state and action:
\vspace{-1mm}
\begin{eqnarray*}
&& {\rm BayesRegret}(T, \pi^{\rm PSRL}, \phi) \\
&\le& \textstyle \sum_{k=1}^{\lceil T/H \rceil} \sum_{h=1}^H 2 (H + 1) \sqrt{\frac{2 \log(4SAT)}{n_k(x_{kh})}} + 2SA + 1 \\
&\le& 10H \sqrt{SAT \log(4SAT)}.
\end{eqnarray*}
\vspace{-1mm}
This completes the proof of Theorem \ref{thm: psrl tight}. \qed




Prior work has designed similar OFU approaches that improve the learning scaling with $S$.
MORMAX \citep{szita2010model} and delayed Q-learning \citep{Strehl2006}, in particular, come with sample complexity bounds that are linear in $S$, and match lower bounds.
But even in terms of sample complexity, these algorithms are not necessarily an improvement over UCRL2 or its variants \citep{dann2015sample}.
For clarity, we compare these algorithms in terms of $T^\pi(\epsilon) := \min \left\{T \mid \frac{1}{T} {\rm BayesRegret}(T, \pi, \phi) \le \epsilon \right\}$.

\begin{table}[!ht]
\vspace{-2mm}
\centering
\label{tab: learning_time_1}
\begin{tabular}{cccc}
DelayQ & MORMAX & UCRL2 & \begin{tabular}[c]{@{}c@{}}PSRL\\ Theorem \ref{thm: psrl tight} \end{tabular} \\
\hline
\vspace{2mm}
$\tilde{O} \left( \frac{H^9 S A}{\epsilon^4} \right)$ &
$\tilde{O} \left( \frac{H^7 S A}{\epsilon^2} \right)$ &
$\tilde{O} \left( \frac{H^2 S^2 A}{\epsilon^2} \right)$ &
$\tilde{O} \left( \frac{H^2 S A}{\epsilon^2} \right)$
\end{tabular}
\vspace{-3mm}
\caption{Learning times compared in terms of $T^\pi(\epsilon)$.}
\end{table}

Theorem 1 implies $T^{\rm PSRL}(\epsilon) = \tilde{O}( \frac{H^2 SA}{\epsilon^2})$.
MORMAX and delayed Q-learning reduces the $S$-dependence of UCRL2, but this comes at the expense of worse dependence on $H$, and
the resulting algorithms are not practical.

\vspace{-1mm}
\subsection{From $H$ to $\sqrt{H}$}

Recent analyses \citep{lattimore2012pac,dann2015sample} suggest that simultaneously reducing the dependence of $H$ to $\sqrt{H}$ may be possible.
They note that ``local value variance'' satisfies a Bellman equation.
Intuitively this captures that if we transition to a bad state $V \simeq 0$, then we cannot transition anywhere much worse during this episode.
This relation means that $\sum_{h=1}^H w^P_h(x_{kh})$ should behave more as if they were independent and grow $O(\sqrt{H})$, unlike our analysis which crudely upper bounds them each in turn $O(H)$.
We present a sketch towards an analysis of Conjecture \ref{conjecture: PSRL regret} in Appendix \ref{app: conjecture}.

\begin{conjecture}
\label{conjecture: PSRL regret}
For any prior over rewards with $\overline{r} \in [0,1]$, additive sub-Gaussian noise and any independent Dirichlet prior over transitions, we conjecture that
\begin{equation}
    \Exp \left[{\rm Regret}(T, \pi^{\rm PSRL}, M^*) \right] = \tilde{O} \left(\sqrt{HSAT} \right),
\vspace{-1mm}
\end{equation}
and that this matches the lower bounds for any algorithm up to logarithmic factors.
\end{conjecture}

The results of \citep{Bartlett2009} adapted to finite horizon MDPs would suggest a lower bound $\Omega(H \sqrt{SAT})$ on the minimax regret for any algorithm.
However, the associated proof is incorrect \citep{osband2016lower}.
The strongest lower bound with a correct proof is $\Omega(\sqrt{HSAT})$ \citep{Jaksch2010}.
It remains an open question whether such a lower bound applies to Bayesian regret over the class of priors we analyze in Theorem \ref{thm: psrl tight}.

One particularly interesting aspect of Conjecture \ref{conjecture: PSRL regret} is that we can construct another algorithm that satisfies the proof of Theorem \ref{thm: psrl tight} but would not satisfy the argument for Conjecture \ref{conjecture: PSRL regret} of Appendix \ref{app: conjecture}.
We call this algorithm Gaussian PSRL, since it operates in a manner similar to PSRL but actually uses the Gaussian sampling we use for the \textit{analysis} of PSRL in its algorithm.

{
\medmuskip=0mu
\thinmuskip=0mu
\thickmuskip=0mu
\begin{algorithm}[!h]
\caption{Gaussian PSRL}
\textbf{Input:} Posterior MAP estimates $\overline{r}_k$, $\hat{P}_k$, visit counts $n_k$\\
\textbf{Output:} Random $Q_{k,h}(s,a) \so Q^*_h(s,a)$ for all $(s,a,h)$

\begin{algorithmic}[1]
\label{alg:gauss_psrl}
\STATE Initialize $Q_{k, H+1}(s,a) \leftarrow 0$ for all $(s,a)$
\FOR{timestep $h=H, H-1, .., 1$}
\STATE \hspace{-2mm} $V_{k,h+1}(s) \leftarrow \max_\alpha Q_{k,h+1}(s, \alpha)$
\STATE \hspace{-2mm} Sample $w_k(s,a,h) \sim N\left(0, \ \frac{(H+1)^2}{\max(n_k(s,a)-2, 1)} \right)$
\STATE \hspace{-2mm} $Q_{k,h}(s,a) \leftarrow \overline{r}_k(s,a) + \hat{P}_k(s,a)^T V + w_k(s,a,h)$ $\forall (s,a)$
\ENDFOR
\end{algorithmic}
\end{algorithm}
\normalsize
}

Algorithm \ref{alg:gauss_psrl} presents the method for sampling random $Q$-values according to Gaussian PSRL, the algorithm then follows these samples greedily for the duration of the episode, similar to PSRL.
Interestingly, we find that our experimental evaluation is consistent with $\tilde{O}(HS\sqrt{AT})$, $\tilde{O}(H\sqrt{SAT})$  and $\tilde{O}(\sqrt{HSAT})$ for UCRL2, Gaussian PSRL and PSRL respectively.

\subsection{An empirical investigation}
\label{sec: empirical}

We now discuss a computational study designed to illustrate how learning times scale with $S$ and $H$, and to empirically investigate Conjecture \ref{conjecture: PSRL regret}.
The class of MDPs we consider involves a long chain of states with $S=H=N$ and with two actions: left and right.
Each episode the agent begins in state $1$.
The optimal policy is to head right at every timestep, all other policies have zero expected reward.
Inefficient exploration strategies will take $\Omega(2^N)$ episodes to learn the optimal policy \citep{osband2014generalization}.

\begin{figure}[!h]
\centering
  \vspace{-3mm}
  \includegraphics[width=.99\linewidth]{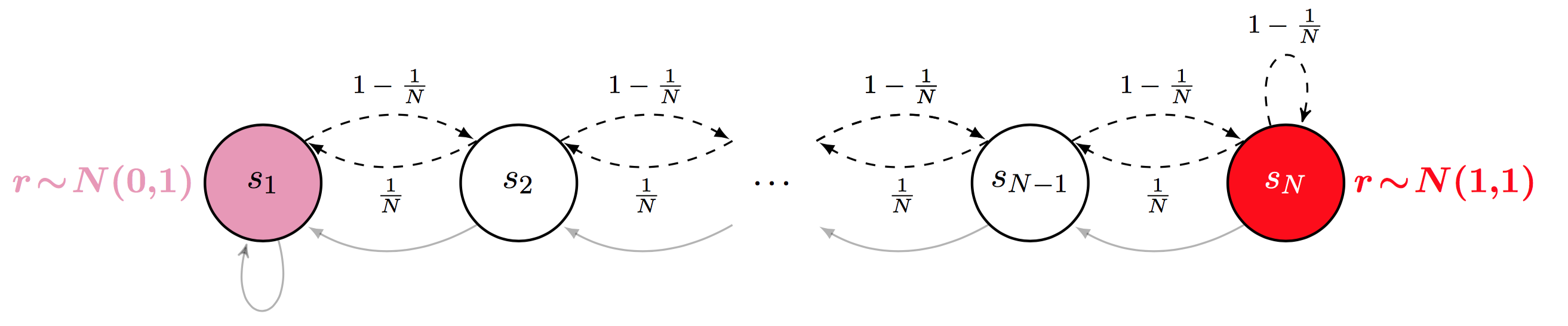}
  \vspace{-4mm}
\caption{MDPs that highlight the need for efficient exploration.}
\label{fig: montezuma chain}
\end{figure}

We evaluate several learning algorithms from ten random seeds and $N=2,..,100$ for up to ten million episodes each.
Our goal is to investigate their empirical performance and scaling.
We believe this is the first ever large scale empirical investigation into the scaling properties of algorithms for efficient exploration.

We highlight results for three algorithms with $\tilde{O}(\sqrt{T})$ Bayesian regret bounds: UCRL2, Gaussian PSRL and PSRL.
We implement UCRL2 with confidence sets optimized for finite horizon MDPs.
For the Bayesian algorithms we use a uniform Dirichlet prior for transitions and $N(0,1)$ prior for rewards.
We view these priors as simple ways to encode very little prior knowledge.
Full details and a link to source code are available in Appendix \ref{app: chain experiments}.

Figure \ref{fig: chain learn} display the regret curves for these algorithms for $N \in \{5, 10, 30, 50\}$.
As suggested by our analysis, PSRL outperforms Gaussian PSRL which outperforms UCRL2.
These differences seems to scale with the length of the chain $N$ and that even for relatively small MDPs, PSRL is many orders of magnitude more efficient than UCRL2.

\begin{figure}[!h]
\centering
  \vspace{-2mm}
  \includegraphics[width=.7\linewidth, height=3.5in]{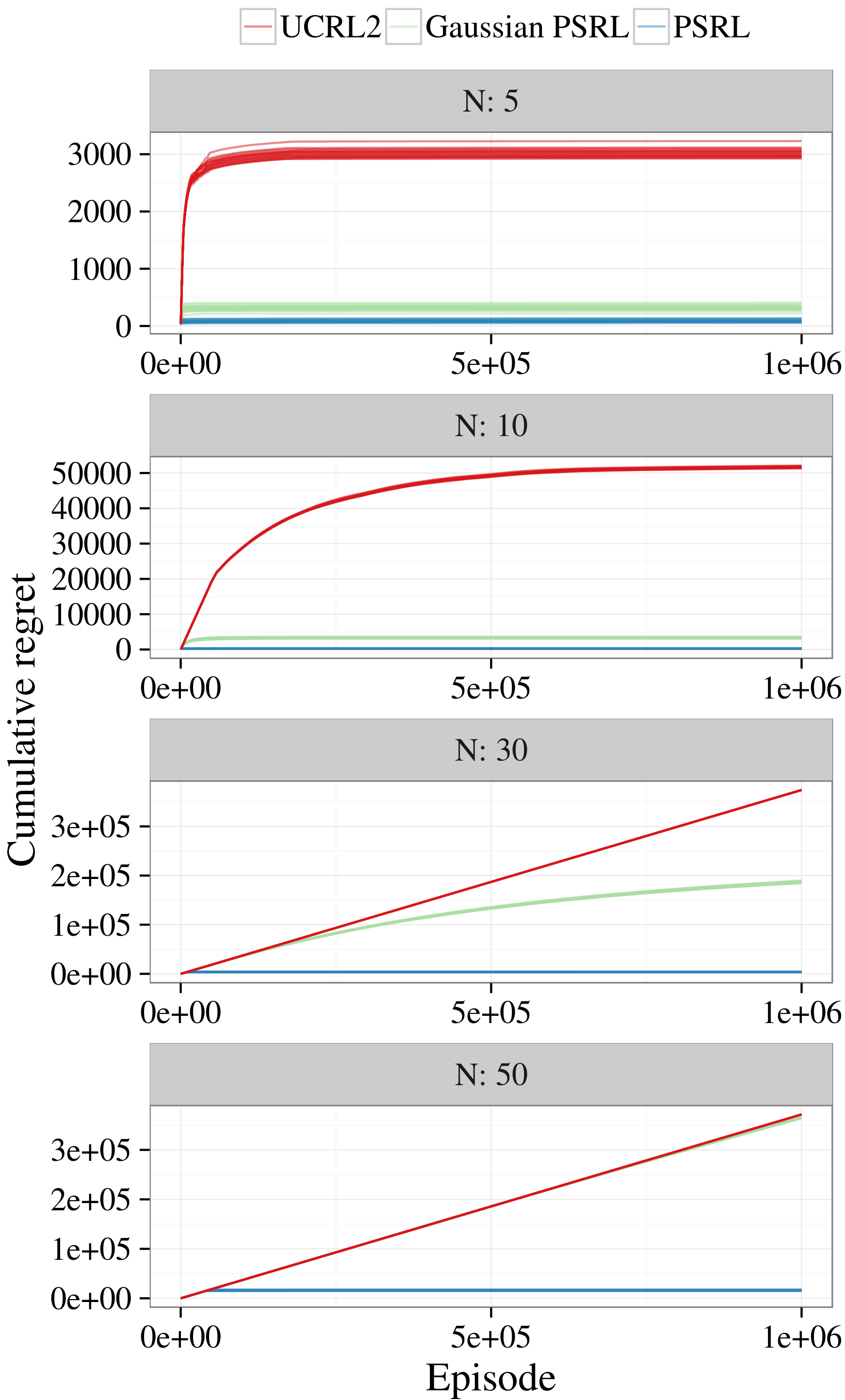}
  \vspace{-4mm}
\caption{PSRL outperforms other methods by large margins.}
\vspace{-2mm}
\label{fig: chain learn}
\end{figure}

We investigate the empirical scaling of these algorithms with respect to $N$.
The results of Theorem \ref{thm: psrl tight} and Conjecture \ref{conjecture: PSRL regret} only bound the Bayesian regret according to the prior $\phi$.
The family of environments we consider in this example are decidedly \textit{not} from this uniform distribution; in fact they are chosen to be as difficult as possible.
Nevertheless, the results of Theorem \ref{thm: psrl tight} and Conjecture \ref{conjecture: PSRL regret} provide remarkably good description for the behavior we observe.

Define \mbox{${\rm learning \ time}(\pi, N) := \min \left\{K \mid \frac{1}{K}\sum_{k=1}^K \Delta_k \le 0.1 \right\}$} for the algorithm $\pi$ on the MDP from Figure \ref{fig: montezuma chain} with size $N$.
For any $B_\pi > 0$, the regret bound $\tilde{O}(\sqrt{B_\pi T})$ would imply $\log({\rm learning \ time})(\pi, N) = B_\pi H \times \log(N) + o(\log(N))$.
In the cases of Figure $\ref{fig: montezuma chain}$ with $H=S=N$ then the bounds $\tilde{O}(HS \sqrt{AT})$, $\tilde{O}(H \sqrt{SAT})$ and $\tilde{O}(\sqrt{HSAT})$ would suggest a slope $B_\pi$ of $5, 4$ and $3$ respectively.

Remarkably, these high level predictions match our empirical results almost exactly, as we show in Figure \ref{fig: chain scale}.
These results provide some support to Conjecture \ref{conjecture: PSRL regret} and even, since the spirit of these environments is similar example used in existing proofs, the ongoing questions of fundamental lower bounds \cite{osband2016lower}.
Further, we note that every single seed of PSRL and Gaussian PSRL learned the optimal policy for every single $N$.
We believe that this suggests it may be possible to extend our Bayesian analysis to provide minimax regret bounds of the style in UCRL2 for suitable choice of diffuse uninformative prior.


\begin{figure}[h!]
\centering
  \vspace{-3mm}
  \includegraphics[width=.9\linewidth]{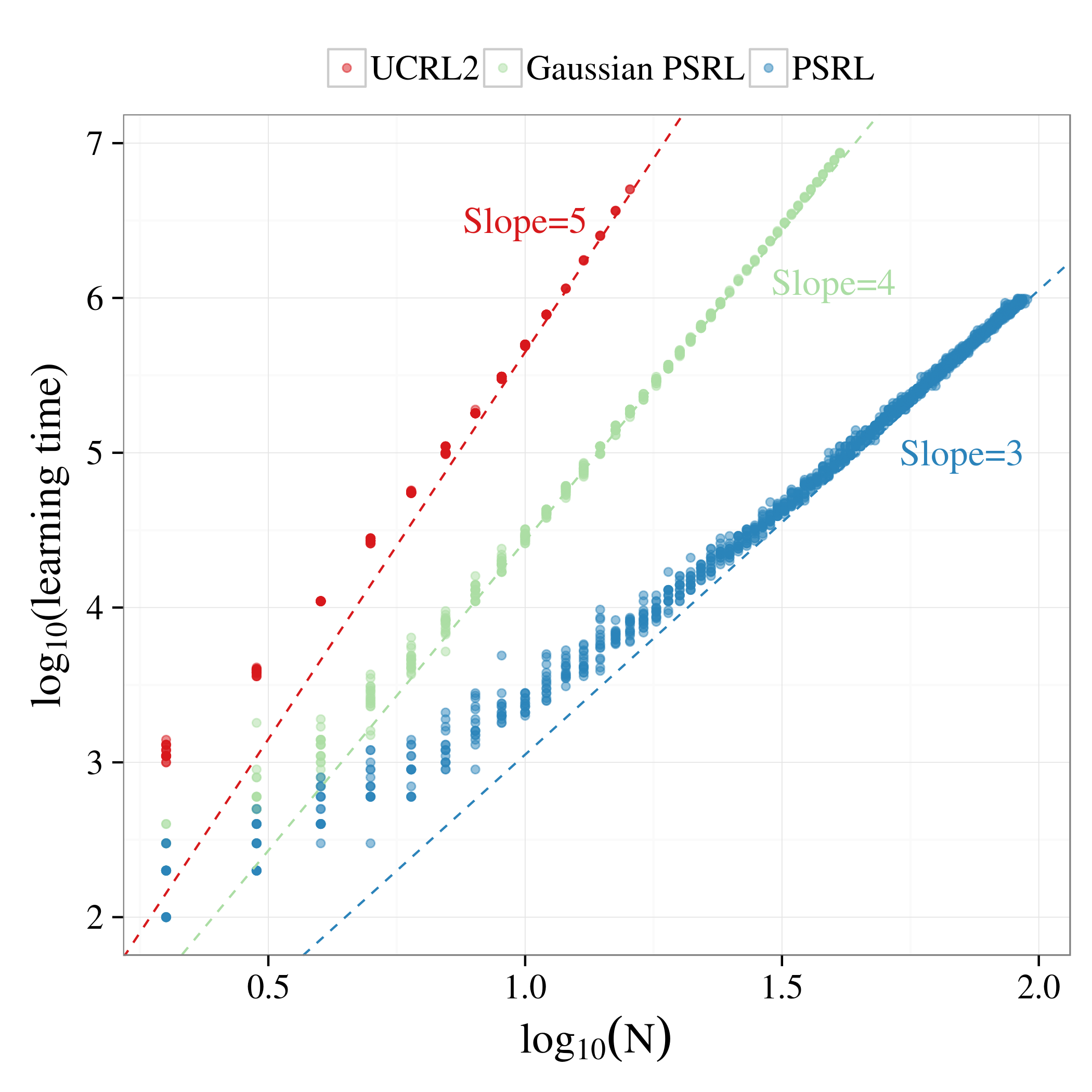}
  \vspace{-5mm}
\caption{Empirical scaling matches our conjectured analysis.}
\label{fig: chain scale}
\vspace{-3mm}
\end{figure}


%

\vspace{-1mm}
\section{Conclusion}

PSRL is orders of magnitude more statistically efficient than UCRL \textit{and} the same computational cost as solving a known MDP.
We believe that analysts will be able to formally specify an OFU approach to RL whose statistical efficiency matches PSRL.
However, we argue that the resulting confidence sets which address both the coupling over $H$ and $S$ may result in a computationally intractable optimization problem.
Posterior sampling offers a computationally tractable approach to statistically efficient exploration.

We should stress that the finite tabular setting we analyze is not a reasonable model for most problems of interest.
Due to the curse of dimensionality, RL in practical settings will require generalization between states and actions.
The goal of this paper is not just to improve a mathematical bound in a toy example (although we do also do that).
Instead, we hope this simple setting can highlight some shortcomings of existing approaches to ``efficient RL'' and provide insight into why algorithms based on sampling may offer important advantages.
We believe that these insights may prove valuable as we move towards algorithms that solve the problem we really care about: synthesizing efficient exploration with powerful generalization.

\nocite{osband2014model,osband2014near,gopalan2014thompson,araya2012near,kolter2009near}


\newpage

\section*{Acknowledgements}
This work was generously supported by DeepMind, a research grant from Boeing, a Marketing Research Award from Adobe, and a Stanford Graduate Fellowship, courtesy of PACCAR.
The authors would like to thank Daniel Russo for many hours of discussion and insight leading to this research, Shipra Agrawal and Tor Lattimore for pointing out several flaws in some early proof steps, anonymous reviewers for their helpful comments and many more colleagues at DeepMind including Remi Munos, Mohammad Azar and more for inspirational conversations.

{
\bibliography{reference}
\bibliographystyle{icml2017}
}


\newpage
\hspace{0.1mm}
\newpage

\appendix
\onecolumn

\begin{center}
\textbf{\Large APPENDICES}
\end{center}

\section{Proof of Lemma \ref{lem: transition conc psrl}}
\label{app: proof_trans}

This section centers around the proof of Lemma \ref{lem: transition conc psrl}, which we reproduce below for completeness.
In the main paper we present a simple sketch for the special case of $S=2$.
We now extend this argument to general MDPs with $S > 2$.
The main strategy for this proof is to proceed via an inductive argument and consider the contribution of each component of $P_k$ in turn.
We will see that, for any choice of component, the resultant random variable is dominated by a matched Gaussian random variable just as in \eqref{eq: w_trans_beta}.

\transconc*

Our analysis of Lemma \ref{lem: transition conc psrl} will rely heavily upon the technical analysis of \citet{osband2017gaussian}.
We first reproduce Lemma 2 from \citet{osband2017gaussian} in terms of stochastic optimism, rather than second order stochastic dominance.

\begin{lemma}[Beta vs Dirichlet dominance]
\label{lem: Dir Beta}
\hspace{0.00000001mm} \newline
Let $X = P^\top v$ for the random variable $P \sim {\rm Dirichlet}(\alpha)$ and constants $v \in \Real^S$ and $\alpha \in \Real_+^S$.
Without loss of generality, assume $v_1 \leq v_2 \leq \cdots \leq v_S$.
Let $\tilde{\alpha} = \sum_{i=1}^s \alpha_i (v_i - v_1) / (v_d-v_1)$ and $\tilde{\beta} = \sum_{i=1}^d \alpha_i (v_d - v_i) / (v_d-v_1)$.
Then, there exists a random variable $\tilde{P} \sim {\rm Beta}(\tilde{\alpha},\tilde{\beta})$ such that,
for $\tilde{X} = \tilde{P} v_d + (1-\tilde{P}) v_1$, $\Exp[\tilde{X} | X] = X$ and $\tilde{X} \so X$.
\end{lemma}


\begin{proof}
Let $\gamma_i = \text{Gamma}(\alpha, 1)$ be independent and identically distributed and let $\overline{\gamma} = \sum_{i=1}^d \gamma_i$, so that
$P \equiv_D \gamma / \overline{\gamma}.$
Let $\alpha_i^0 = \alpha_i (v_i - v_1) / (v_d-v_1)$ and $\alpha_i^1 = \alpha_i (v_d - v_i) / (v_d-v_1)$ so that $\alpha = \alpha^0 + \alpha^1.$
Define independent random variables
$\gamma^0 \sim \text{Gamma}(\alpha_i^0, 1)$ and $\gamma^1 \sim \text{Gamma}(\alpha_i^1, 1)$ so that
$\gamma \equiv_D \gamma^0 + \gamma^1.$

Take $\gamma^0$ and $\gamma^1$ to be independent, and couple these variables with $\gamma$ so that
$\gamma = \gamma^0 + \gamma^1.$
Note that $\tilde{\beta} = \sum_{i=1}^d \alpha^0_i$ and  $\tilde{\alpha} = \sum_{i=1}^d \alpha^1_i$.
Let $\overline{\gamma}^0 = \sum_{i=1}^d \gamma^0_i$ and $\overline{\gamma}^1 = \sum_{i=1}^d \gamma^1_i$, so that
$1-\tilde{P} \equiv_D \overline{\gamma}^0 / \overline{\gamma}$ and
$\tilde{P} \equiv_D \overline{\gamma}^1 / \overline{\gamma}.$
Couple these variables so that
$1-\tilde{P} = \overline{\gamma}^0 / \overline{\gamma}$ and
$ \tilde{P} = \overline{\gamma}^1 / \overline{\gamma}.$
We can now say,
\begin{eqnarray*}
\Exp[\tilde{X} | X]
&=& \Exp[(1- \tilde{P}) v_1 + \tilde{P} v_d | X]
= \Exp\left[\frac{v_1 \overline{\gamma}^0}{\overline{\gamma}} + \frac{v_d \overline{\gamma}^1}{\overline{\gamma}} \Big| X\right] \\
&=& \Exp\left[\Exp\left[\frac{v_1 \overline{\gamma}^0 + v_d \overline{\gamma}^1}{\overline{\gamma}} \Big| \gamma, X\right] \Big| X \right]
= \Exp\left[\frac{v_1 \Exp[\overline{\gamma}^0 | \gamma] + v_d \Exp[\overline{\gamma}^1 | \gamma]}{\overline{\gamma}} \Big| X \right] \\
&=& \Exp\left[\frac{v_1 \sum_{i=1}^d \Exp[\gamma^0_i | \gamma_i] + v_d \sum_{i=1}^dxp[\gamma^1_i | \gamma_i]}{\overline{\gamma}} \Big| X \right] \\
&\stackrel{\text{(a)}}{=}& \Exp\left[\frac{v_1 \sum_{i=1}^d \gamma_i \alpha_i^0 / \alpha_i + v_d \sum_{i=1}^d \gamma_i \alpha_i^1/\alpha_i}{\overline{\gamma}} \Big| X \right] \\
&=& \Exp\left[\frac{v_1 \sum_{i=1}^d \gamma_i (v_i - v_1) + v_d \sum_{i=1}^d \gamma_i (v_d - v_i)}{\overline{\gamma} (v_d - v_1)} \Big| X \right] \\
&=& \Exp\left[\frac{\sum_{i=1}^d \gamma_i  v_i}{\overline{\gamma}} \Big| X \right]
= \Exp\left[\sum_{i=1}^d p_i  v_i \Big| X \right]
= X,
\end{eqnarray*}
where (a) follows from elementary properties of Gamma distribution \cite{osband2017gaussian}.
Therefore, $\tilde{X}$ is a mean-preserving spread of $X$ and so by definition of stochastic optimism $\tilde{X} \so X$.
\end{proof}

\newpage

Next, consider any fixed $P_k(x_{kh})$ and let $R_k$ and $P_k(x \neq x_{kh})$ vary in any arbitrary way to maximize the variation from transition $ w^P_k(x_{kh}) =  (P_k(x_{kh}) - \hat{P}(x_{kh}))^T V^k_{kh+1}$ through their effects on the future value $V^k_{kh+1} \in [0,H]^S$.
We can then upper bound the deviation from transitions by the deviation under the worst possible $v \in [0,H]^S$.

\begin{equation}
\label{eq: w_p bound}
  w^P_h(x_{kh})
  \le \max_{R_k, P_k(x \neq x_{kh})} (P_k(x_{kh}) - \hat{P}_k(x_{kh}))^T V^k_{kh+1}
  \le \max_{v \in [0,H]^S} (P_k(x_{kh}) - \hat{P}_k(x_{kh}))^T v .
\end{equation}

We can then apply Lemma \ref{lem: Dir Beta} to \eqref{eq: w_p bound}: for any possible value of $v \in [0,H]^S$ there is a matched Beta random variable that is stochastically optimistic for $w^P_h(x_{kh})$.
This means that we can then apply Lemma \ref{lem:Gaussian-Dirichlet} to show that there is a matched $X \sim \left(0, \frac{H^2}{\alpha^T \Ind}\right) \so w^P_h(x_{kh})$.
To complete the proof of Lemma \ref{lem: transition conc psrl} we apply the Gaussian tail concentration Lemma \ref{lem: tail}.


\section{Conjecture of $\tilde{O}(\sqrt{HSAT})$ bounds}
\label{app: conjecture}

The key remaining loose piece of our analysis concerns the summation $\sum_{h=1}^H w^P_h(x_{kh})$.
Our current proof of Theorem \ref{thm: psrl tight} bounds each $w^P_h(x_{kh})$ independently.
Each term is $\tilde{O}(\sqrt{\frac{H}{n_k(x_{kh})}})$ and we bound the resulting sum $\tilde{O}(H \sqrt{\frac{H}{n_k(x_{kh})}})$.
However, this approach is very loose and pre-supposes that \textit{each} timestep could be maximally bad during a single episode.
To repeat our geometric intuition, we have assumed a worst-case hyper-rectangle over all timesteps $H$ when the actual geometry should be an ellipse.
We therefore suffer an additional term of $\tilde{O}(\sqrt{H})$ in exactly the style of Figure \ref{fig: rect conf}.

In fact, it is not even possible to sequentially get the ``worst-case'' transitions $O(H)$ at each and every timestep during an episode, since once your sample gets one such transition then there will be no more future value to deplete.
Rather than just being independent per timestep, which would be enough for us to end up with an $\tilde{O}(\sqrt{H})$ saving, they actually have some kind of anti-correlation property through the law of total variance.
A very similar observation is used by recent analyses in the sample complexity setting \cite{lattimore2012pac} and also finite horizon MDPs \citep{dann2015sample}.
This seems to suggest that it should be possible to combine the insights of Lemma \ref{lem: transition conc psrl} with, for example, Lemma 4 of \cite{dann2015sample} to remove \textit{both} the $\sqrt{S}$ \textit{and} the $\sqrt{H}$ from our bounds to prove Conjecture \ref{conjecture: PSRL regret}.

We note that this informal argument would \textit{not} apply Gaussian PSRL, since it generates $w^P$ from some Gaussian posterior which does not satisfy the Bellman operators.
Therefore, we should be able to find some evidence for this conjecture if we find domains where UCRL, Gaussian PSRL and PSRL all demonstrate their (unique) predicted scalings.
We present some evidence of this effect in Section \ref{sec: empirical} and find that that our empirical results are consistent with this conjecture.

\section{Estimation experiments}
\label{app: estimation_experiments}

In this section we expand upon the simple examples given by Section \ref{sec: bandit_computation} to a full decision problem with two actions.
We define an MDP similar to Figures \ref{fig: bandit_s} and \ref{fig: bandit_h} but now with two actions.
The first action is identical to Figure \ref{fig: bandit_s}, but the second action modifies the transition probabilities to favor the rewarding states with probability $0.6 / N$ and assigning only $0.4 / N$ to the non-rewarding states.

We now investigate the \textit{regret} of several learning algorithms which we adapt to this setting.
These algorithms are based upon BEB \citep{kolter2009near}, BOLT \citep{araya2012near}, $\epsilon$-greedy with $\epsilon=0.1$, Gaussian PSRL (see Algorithm \ref{alg:gauss_psrl}), Optimistic PSRL (which takes $K=10$ samples and takes the maximum over sampled Q-values similar to BOSS \citep{asmuth2009bayesian}), PSRL \citep{Strens00}, UCFH \citep{dann2015sample} and UCRL2 \citep{Jaksch2010}.
We link to the full code for implementation in Appendix \ref{app: chain experiments}.

\begin{figure}[!h]
\centering
  \includegraphics[width=.7901\linewidth]{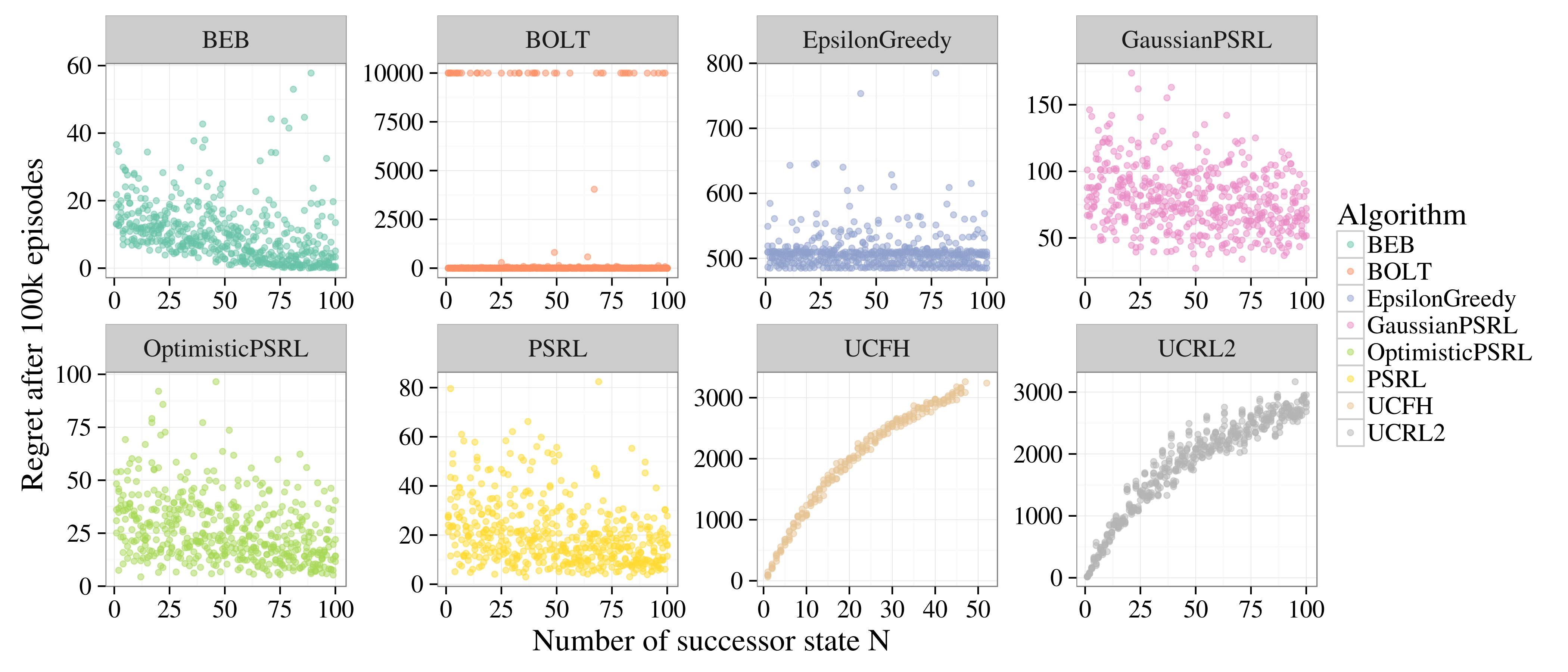}
\vspace{-3mm}
\caption{Known rewards $R$ and unknown transitions $P$, similar to Figure \ref{fig: bandit_knownR}.}
\label{fig: bandit regret knownR}
\end{figure}

\begin{figure}[!h]
\centering
  \includegraphics[width=.7901\linewidth]{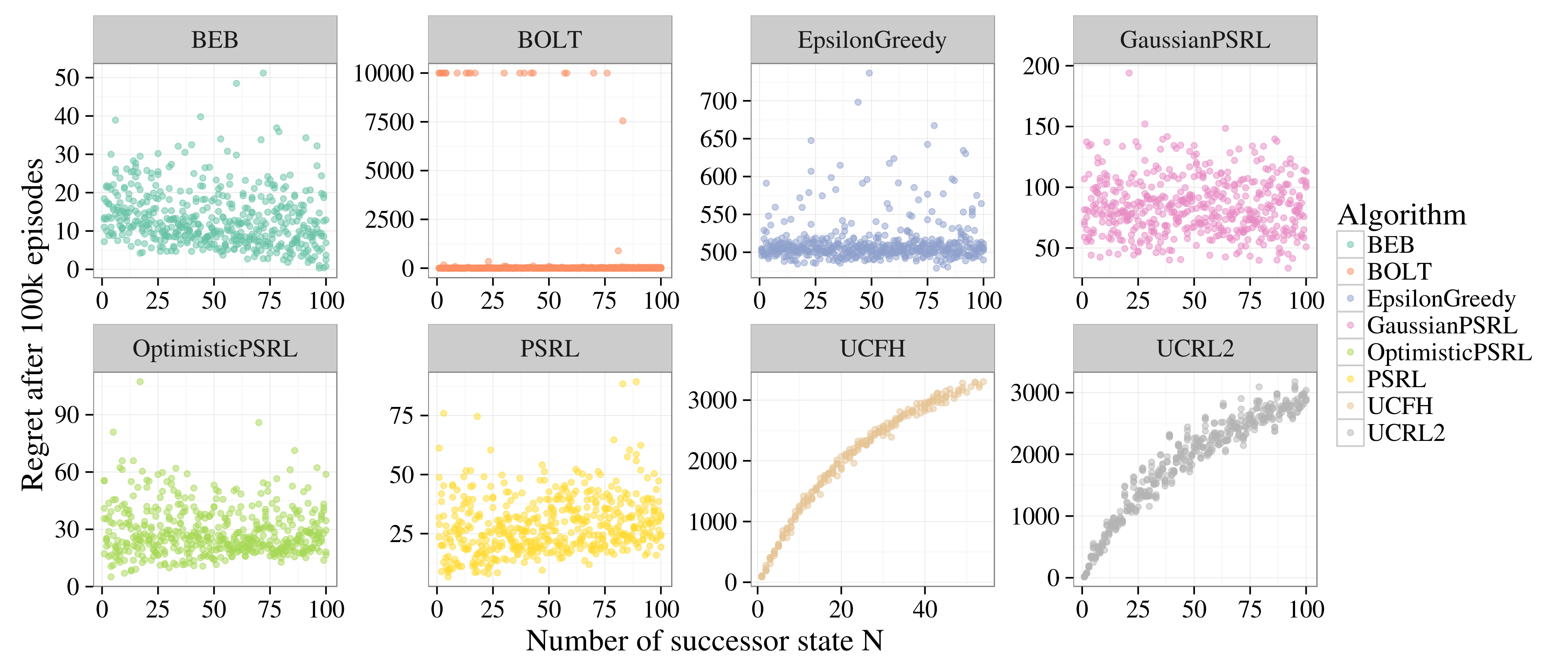}
\vspace{-3mm}
\caption{Unknown rewards $R$ and known transitions $P$, similar to Figure \ref{fig: bandit_knownP}.}
\label{fig: bandit regret knownP}
\end{figure}

We see that the loose estimates in OFU algorithms from Figures \ref{fig: bandit_knownR} and \ref{fig: bandit_knownP} lead to bad performance in a decision problem.
This poor scaling with the number of successor states $N$ occurs when \textit{either} the rewards or the transition function is unknown.
We note that in stochastic environments the PAC-Bayes algorithm BOLT, which relies upon optimistic fake prior data, can sometimes concentrate too quickly and so incur the maximum linear regret.
In general, although BOLT is PAC-Bayes, it concentrates too fast to be PAC-MDP just like BEB \citep{kolter2009near}.

In Figure \ref{fig: bandit regret epLen} we see a similar effect as we increase the episode length $H$.
We note the second order UCFH modification improves upon UCRL2's miscalibration with $H$, as is reflected in their bounds \citep{dann2015sample}.
We note that both BEB and BOLT scale poorly with the horizon $H$.

\begin{figure}[!h]
\centering
  \includegraphics[width=.7901\linewidth]{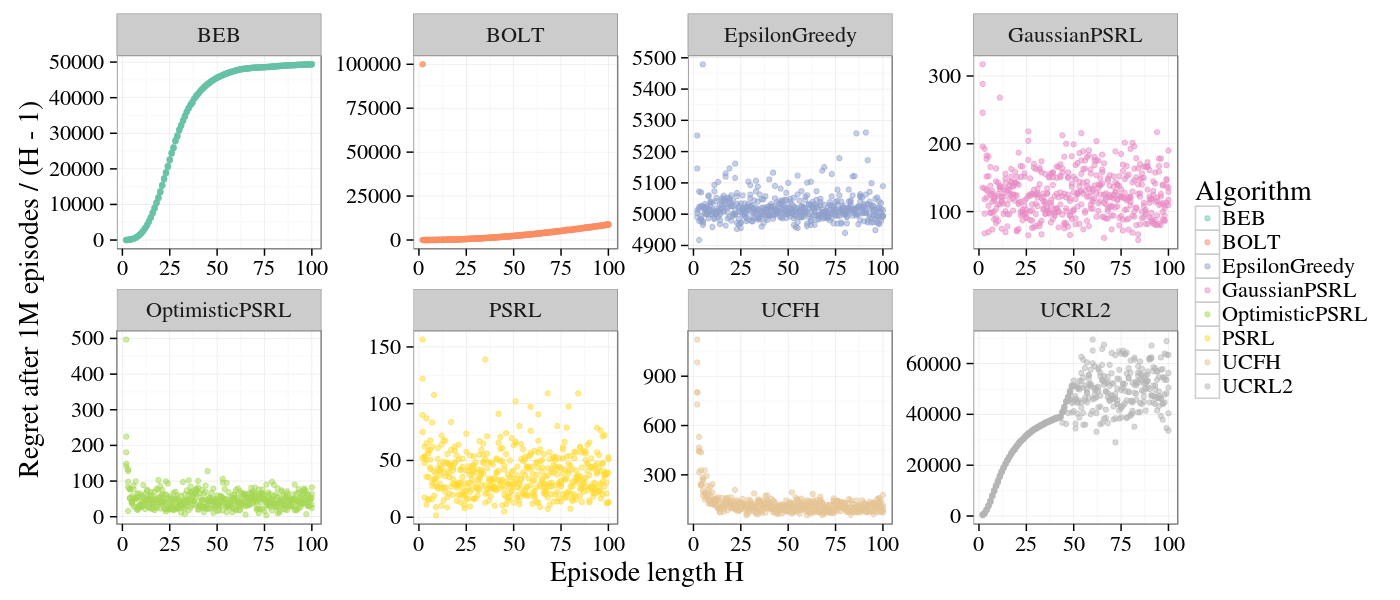}
\vspace{-3mm}
\caption{Unknown rewards $R$ and transitions $P$, similar to Figure \ref{fig: bandit_epLen}.}
\label{fig: bandit regret epLen}
\end{figure}

\section{Chain experiments}
\label{app: chain experiments}

All of the code and experiments used in this paper are available in full on github.
As per the review request we have removed the link to this code, but instead include an anonymized excerpt of the some of the code in our submission file.
We hope that researchers will find this simple codebase useful for quickly prototyping and experimenting in tabular reinforcement learning simulations.

In addition to the results already presented we also investigate the scaling of similar Bayesian learning algorithms BEB \cite{kolter2009near} and BOLT \cite{araya2012near}.
We see that neither algorithms scale as gracefully as PSRL, although BOLT comes close.
However, as observed in Appendix \ref{app: estimation_experiments}, BOLT can perform poorly in highly stochastic environments.
BOLT also requires $S$-times more computational cost than PSRL or BEB.
We include these algorithms in Figure \ref{fig: chain scale full}.

\begin{figure}[!h]
\centering
  \includegraphics[width=.7901\linewidth]{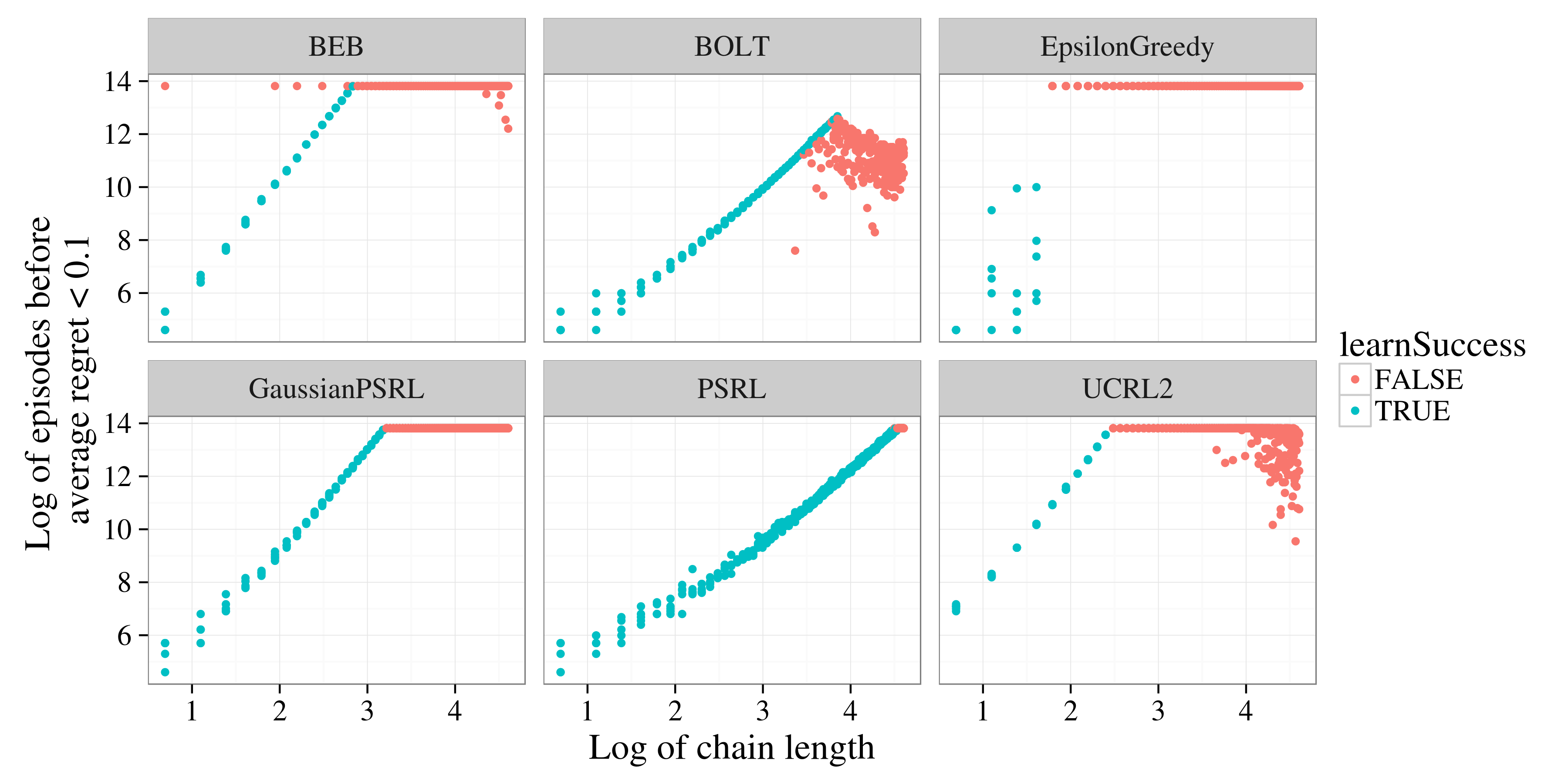}
\vspace{-3mm}
\caption{Scaling of more learning algorithms.}
\label{fig: chain scale full}
\end{figure}

\subsection{Rescaling confidence sets}
It is well known that provably-efficient OFU algorithms can perform poorly in practice.
In response to this observation, many practitioners suggest rescaling confidence sets to obtain better empirical performance \citep{szita2010model,araya2012near,kolter2009near}.
In Figure \ref{fig: chain scaling} we present the performance of several algorithms with confidence sets rescaled $\in \{0.01, 0.03, 0.1, 0.3, 1 \}$.
We can see that rescaling for tighter confidence sets can sometimes give better empirical performance.
However, it does not change the fundamental scaling of the algorithm.
Also, for aggressive scalings some seeds may not converge at all.

\begin{figure}[!h]
\centering
  \includegraphics[width=.7901\linewidth]{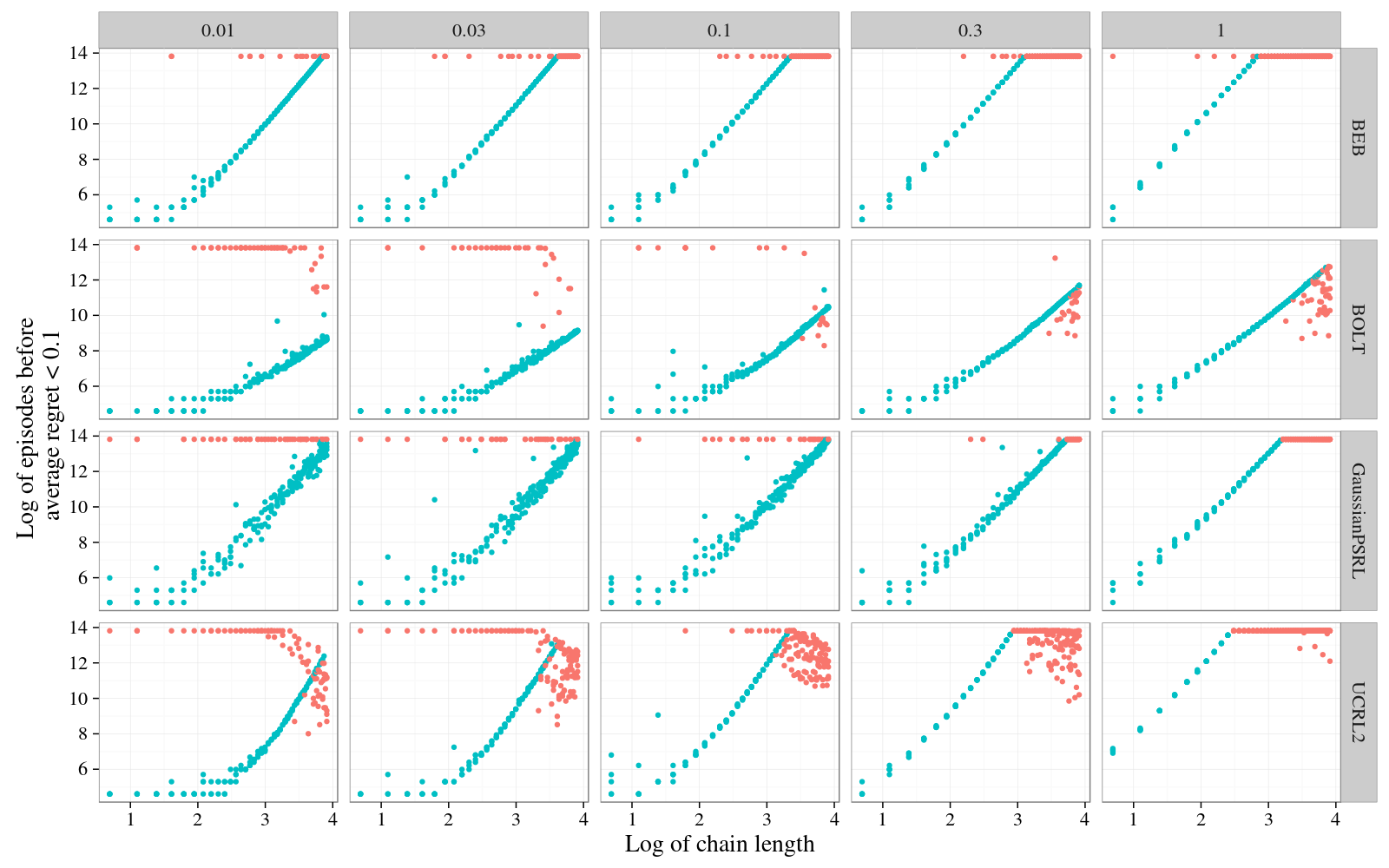}
\vspace{-3mm}
\caption{Rescaled proposed algorithms for more aggressive learning.}
\label{fig: chain scaling}
\end{figure}

\subsection{Prior sensitivities}

We ran all of our Bayesian algorithms with uninformative independent priors for rewards and transitions.
For rewards, we use $\overline{r}(s,a) \sim N(0,1)$ and updated as if the observed noise were Gaussian with precision $\tau = \frac{1}{\sigma^2} = 1$.
For transitions, we use a uniform Dirichlet prior $P(s,a) \sim {\rm Dirchlet}(\alpha)$.
In Figures \ref{fig: prior GaussianPSRL} and \ref{fig: prior PSRL} we examine the performance of Gaussian PSRL and PSRL on a chain of length $N=10$ as we vary $\tau$ and $\alpha = \alpha_0 \Ind$.

\begin{figure}[!h]
\centering
  \includegraphics[width=.7901\linewidth]{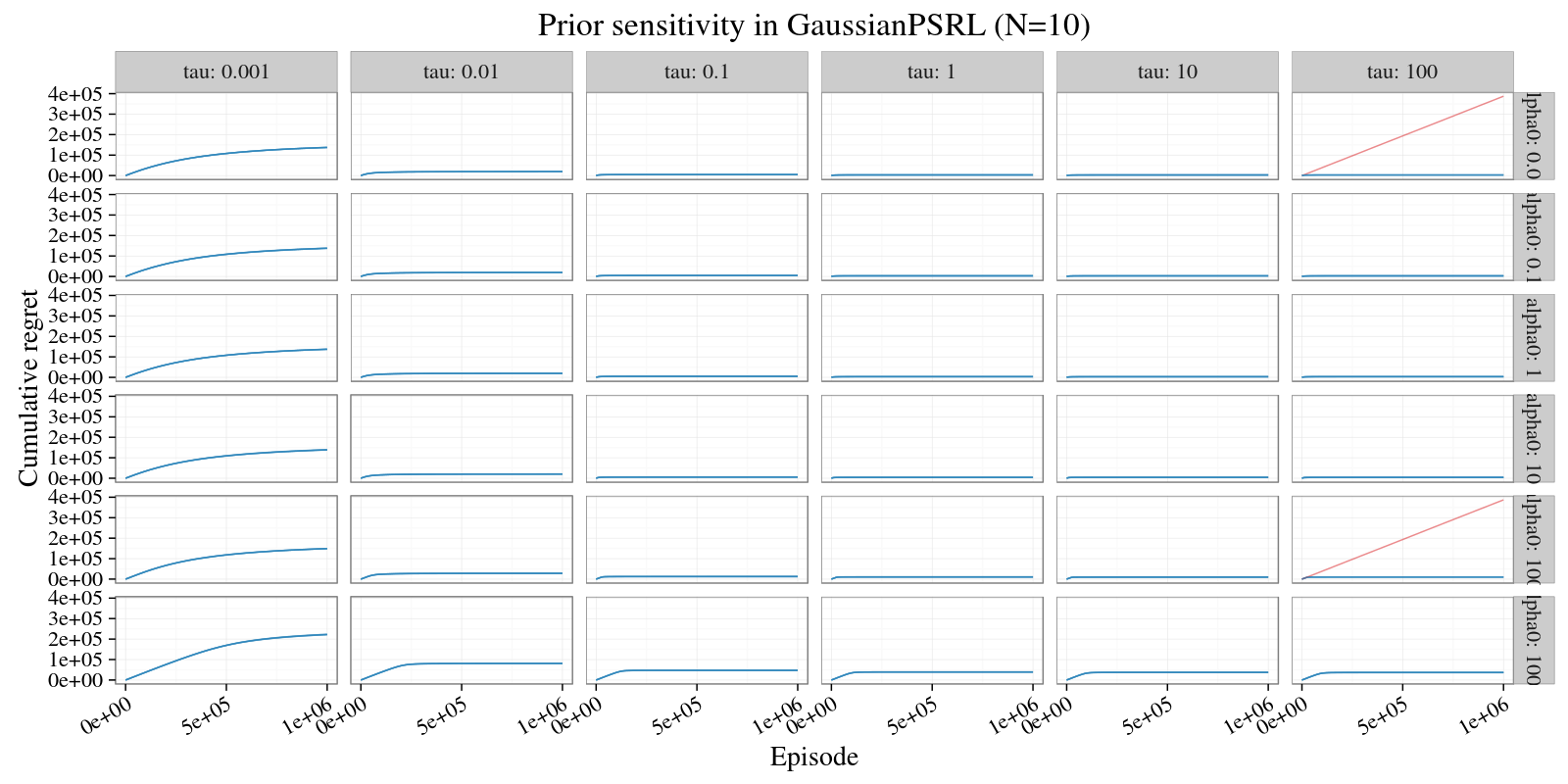}
\vspace{-3mm}
\caption{Prior sensitivity in Gaussian PSRL.}
\label{fig: prior GaussianPSRL}
\end{figure}

\begin{figure}[!h]
\centering
  \includegraphics[width=.7901\linewidth]{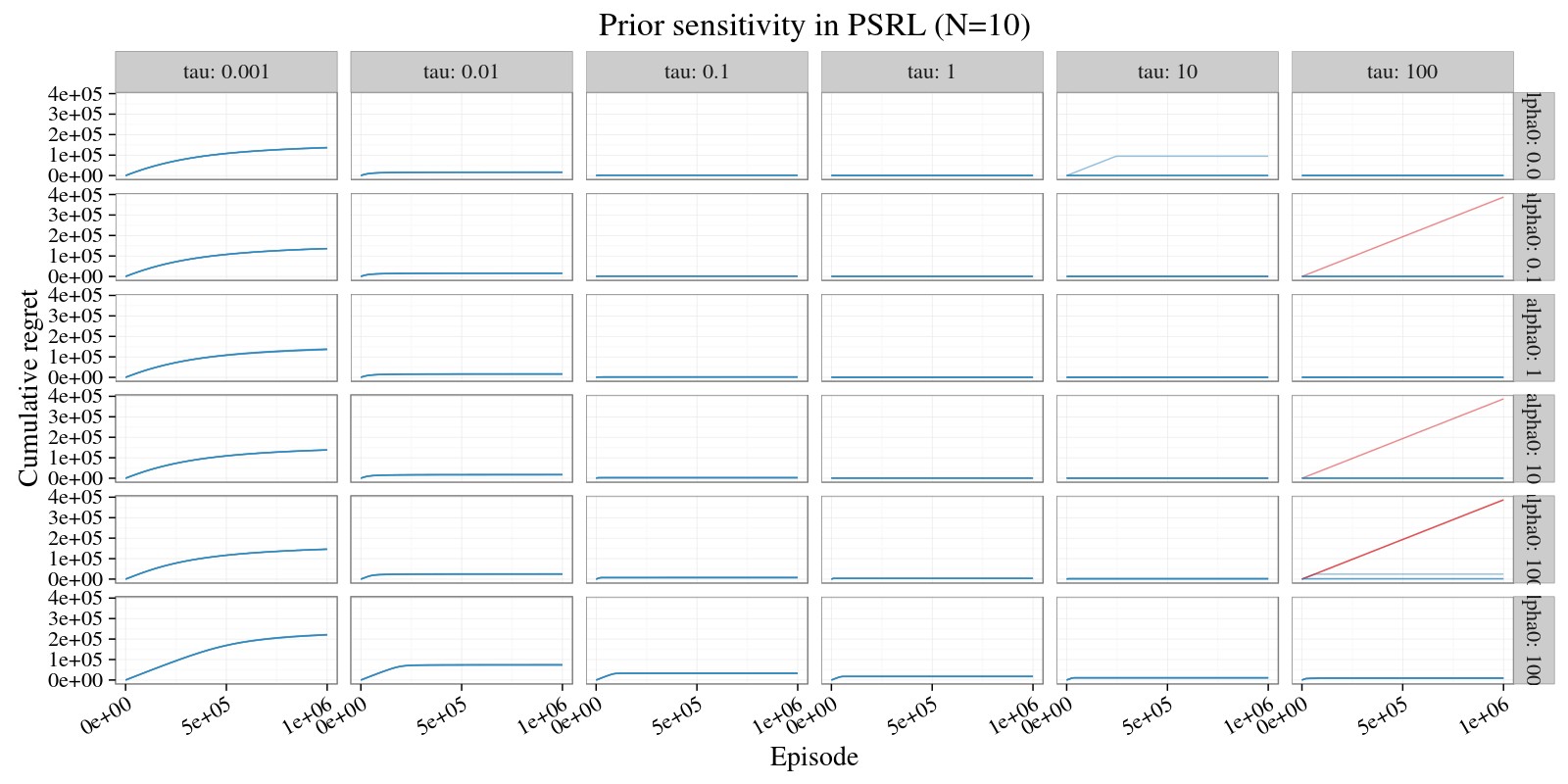}
\vspace{-3mm}
\caption{Prior sensitivity in PSRL.}
\label{fig: prior PSRL}
\end{figure}

We find that both of the algorithms are extremely robust over several orders of magnitude.
Only large values of $\tau$ (which means that the agent updates it reward prior too quickly) caused problems for some seeds in this environment.
Developing a more clear frequentist analysis of these Bayesian algorithms is a direction for important future research.

\subsection{Optimistic posterior sampling}

We compare our implementation of PSRL with a similar optimistic variant which samples $K \ge 1$ samples from the posterior and forms the optimistic $Q$-value over the envelope of sampled $Q$-values.
This algorithm is sometimes called ``optimistic posterior sampling'' \cite{fonteneau2013optimistic}.
We experiment with this algorithm over several values of $K$ but find that the resultant algorithm performs very similarly to PSRL, but at an increased computational cost.
We display this effect over several magnitudes of $K$ in Figures \ref{fig: psrl nsamp reg} and \ref{fig: psrl nsamp}.

\begin{figure}[!h]
\centering
  \includegraphics[width=.7901\linewidth]{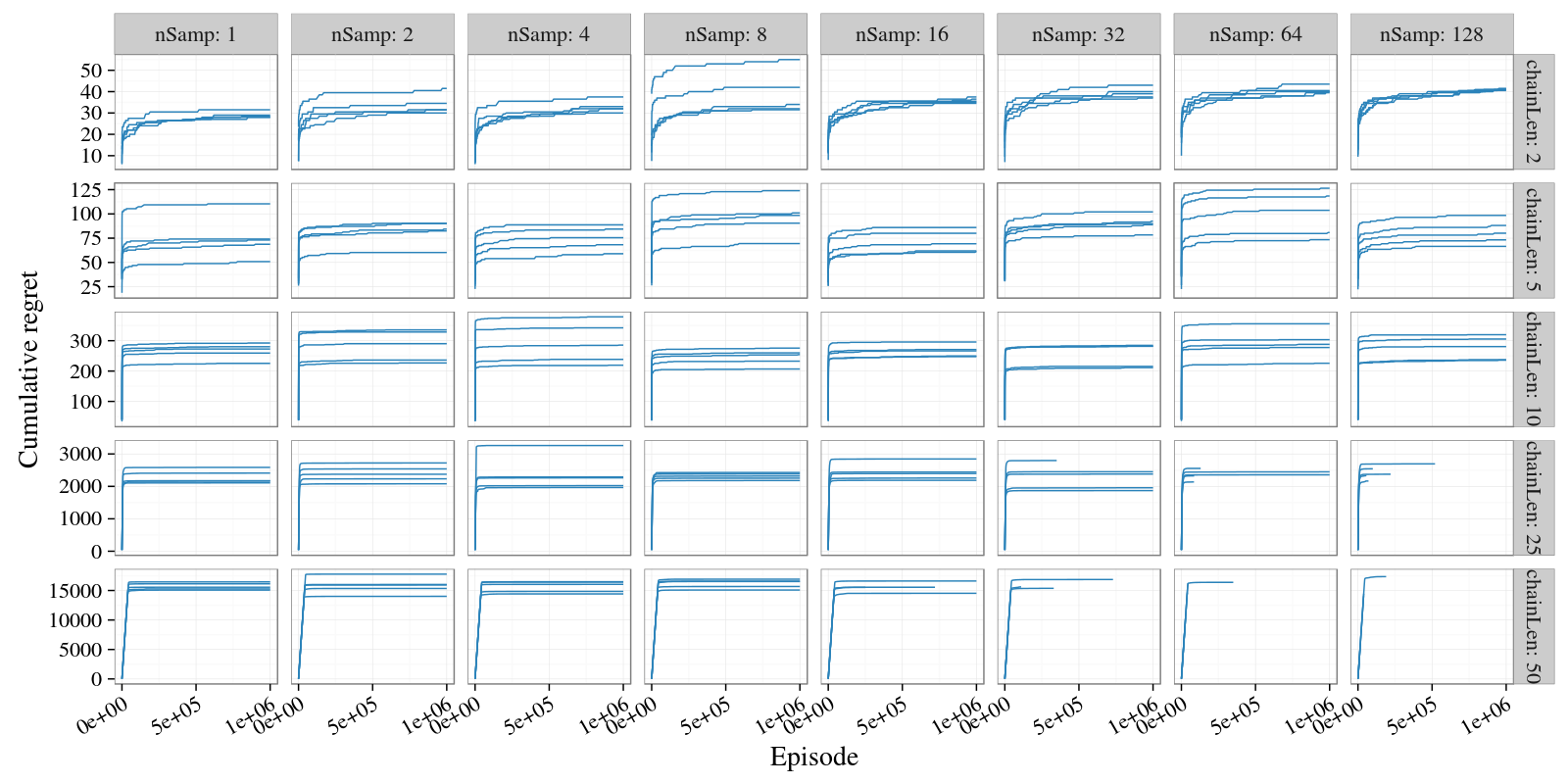}
\vspace{-3mm}
\caption{PSRL with multiple samples is almost indistinguishable.}
\label{fig: psrl nsamp reg}
\end{figure}

\begin{figure}[!ht]
\centering
  \includegraphics[width=.7901\linewidth]{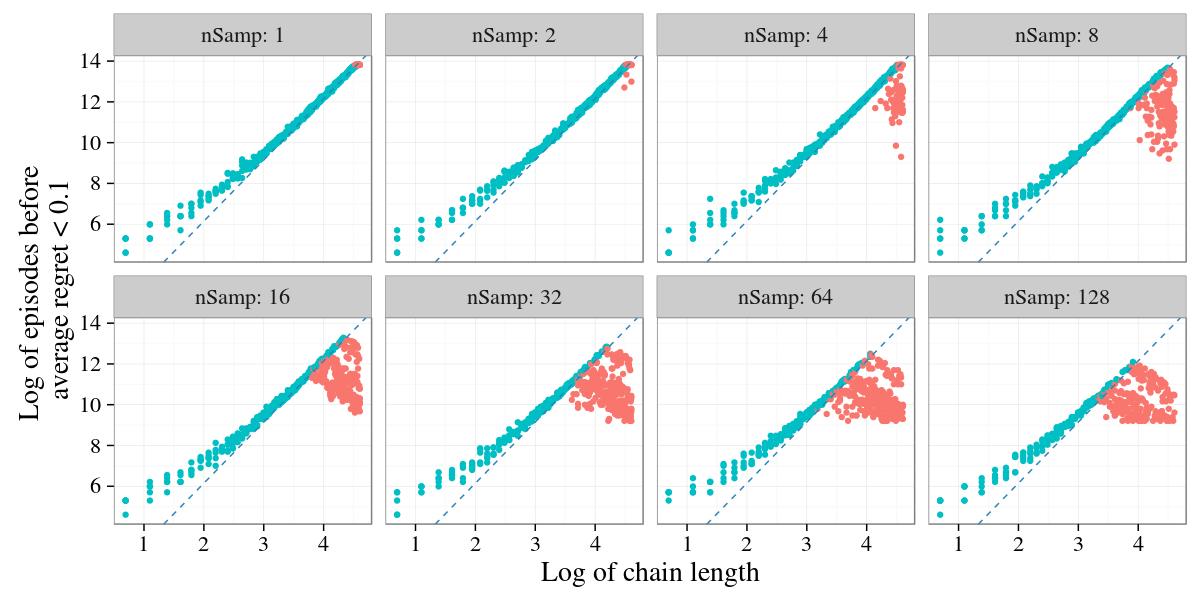}
\vspace{-3mm}
\caption{PSRL with multiple samples is almost indistinguishable.}
\label{fig: psrl nsamp}
\end{figure}

This algorithm ``Optimistic PSRL'' is spiritually very similar to BOSS \citep{asmuth2009bayesian} and previous work had suggested that $K > 1$ could lead to improved performance.
We believe that an important difference is that PSRL, unlike Thompson sampling, should not resample every timestep but previous implementations had compared to this faulty benchmark \citep{fonteneau2013optimistic}.

\end{document}